%% file: groupICMarxiv_v2.tex
\documentclass[letterpaper]{article}
\usepackage{proceed2e}
\usepackage[margin=1in]{geometry}
\usepackage{tikz}
\usepackage{standalone}
\usepackage{mathtools}
\usepackage{pgf}
\usepackage[square]{natbib}
\usetikzlibrary{shapes}
\usetikzlibrary{arrows.meta}
\usetikzlibrary{matrix}
\usetikzlibrary{calc}

\usepackage{url}

\usepackage{color}\usepackage{xcolor} 
\colorlet{myorange}{red!30!yellow}
\usepackage{amssymb}
\usepackage{bbm}
\usepackage{etoolbox}
\usepackage{amsmath}

\usepackage{comment}
\specialcomment{commenth}{\bgroup\textcolor{red}\bgroup }{\egroup\egroup}

\usepackage{amsbsy}
\usepackage{amssymb}
\usepackage{graphicx}
\usepackage[normalem]{ulem}

\usepackage{caption}
\usepackage{subcaption}
\makeatletter
\DeclareRobustCommand{\qed}{%
	\ifmmode 
	\else \leavevmode\unskip\penalty9999 \hbox{}\nobreak\hfill
	\fi
	\quad\hbox{\qedsymbol}}
\newcommand{\openbox}{\leavevmode
	\hbox to.77778em{%
		\hfil\vrule
		\vbox to.675em{\hrule width.6em\vfil\hrule}%
		\vrule\hfil}}
\newcommand{\qedsymbol}{\openbox}
\newenvironment{proof}[1][\proofname]{\par
	\normalfont
	\topsep6\p@\@plus6\p@ \trivlist
	\item[\hskip\labelsep\itshape
	#1.]\ignorespaces
}{%
	\qed\endtrivlist
}
\newcommand{\proofname}{Proof}
\makeatother

\newtheorem{thm}{\protect\theoremname}
\newtheorem{lem}[thm]{Lemma}
\newtheorem{defn}[thm]{\protect\definitionname}
\newtheorem{prop}[thm]{\protect\propositionname}
\newtheorem{conj}[thm]{\protect\propositionname}
\newtheorem{post}{Postulate}
\newtheorem{example}[thm]{\protect\examplename}  

\newcommand{\T}{\mathcal{T}}
\newcommand{\M}{\mathcal{M}}
\newcommand{\X}{\mathcal{X}}
\newcommand{\G}{\mathcal{G}}
\newcommand{\A}{\mathcal{A}}
\newcommand{\I}{\mathcal{I}}
\newcommand{\Ones}{{\bf\mathbbm{1}}}
\newcommand{\ones}{{\bf 1}}
\newcommand{\tr}{\operatorname{tr}}
\newcommand{\inv}[1]{{#1^{-1}}}


	\newcommand\domi[1]
	{\marginpar{\tiny\textcolor{brown}{Dominik: #1}}}
	{{}}
	\newcommand\marek[1]{\marginpar{\tiny\textcolor{cyan}{Marek: #1}}}
	\newcommand\michel[1]{\marginpar{\tiny\textcolor{red}{Michel: #1}}}
	\renewcommand\michel[1]{}

\makeatother
\providecommand{\definitionname}{Definition}
\providecommand{\examplename}{Example}
\providecommand{\propositionname}{Proposition}
\providecommand{\theoremname}{Theorem}

\usepackage{times}

\usepackage{algorithm}
\usepackage{algorithmic}

\usepackage{hyperref}
\usepackage{cleveref}
\crefname{figure}{Fig.}{Figs.}
\crefname{prop}{proposition}{propositions}

\author{Michel Besserve$^{1,2,3}$, Naji Shajarisales$^1$, Bernhard Sch\"olkopf$^1$, Dominik Janzing$^1$\\1. Max Planck Institute for Intelligent Systems, T\"ubingen, Germany\\2. Max Planck Institute for Biological Cybernetics, T\"ubingen, Germany \\ 3. Max Planck ETH Center for Learning Systems}

\begin{document} 
	\excludecomment{commenth}
\excludecomment{comment}
\title{Group invariance principles for causal generative models}
\maketitle


\vskip 0.3in

\begin{abstract}
The postulate of independence of cause and mechanism (ICM) has recently led to several new causal discovery algorithms. The interpretation of independence and the way it is utilized, however, varies across these methods. Our aim in this paper is to propose a group theoretic framework for ICM to unify and generalize these approaches. In our setting, the cause-mechanism relationship is assessed by comparing it against a null hypothesis through the application of random generic group transformations.  
We show that the group theoretic view provides a very general tool to study the structure of data generating mechanisms with direct applications to machine learning. 
\end{abstract} 
\input{introICML17}

\input{occlusionarxivv1}
\input{frameworkarxivv1}
\input{gt_viewICML17}
\input{groupNMFarxivv1}


\bibliographystyle{plainnat}
{
	\bibliography{causalityv2}
}
\input{SI_ICML17bis}

\end{document}

%% file: introICML17.tex
\section{Introduction}

Inferring causal relationships from empirical data is a challenging
problem with major applications. While the problem of inferring such relations between arbitrarily many
random variables (RVs) has been extensively addressed 
via conditional statistical independences in graphical models  \citep{spirtes2000causation,pearl2000causality},
several limitations of this framework have motivated the search for
new perspectives on causal inference. A major contribution to this
line of research is the postulate of Independence of Cause and Mechanism
(ICM) \citep{janzing2010causal,LemeireJ2012,anticausal}, which assumes
that causes and mechanisms are chosen independently by Nature and thus $P({\rm cause})$ and
$P({\rm effect}|{\rm cause})$ do not contain information about each other. Here, ``no information'' is either meant in the sense
of algorithmic independence \citep{janzing2010causal,LemeireJ2012} or
in the sense that semi-supervised learning does not work \citep{anticausal}.
A major interest of this framework is the development of several causal inference algorithms for cause-effect pairs
\citep{icml2010_062,frW_UAI,daniusis2012inferring,
janzing2012information,shajarisale2015,SgoJanHenSch15}; however, results in \citet{anticausal} also suggest it can be exploited in broader settings, providing guiding principle for the study of learning algorithms. Each of these methods addresses the causal inference problem with specific models,
and are thus usable only for a restricted set of applications. Principled
ways to generalize them to address new problems are yet unknown. In particular, it is unclear how the notion of ``independence'' should be defined for a given domain, and how it could impact the results. One conceptual difficulty of the ICM framework is that independence is assessed between two objects of
different nature: the input (or cause) and the mechanism; moreover, the appropriate notion of independence is not the usual statistical independence of RVs. 

In this paper, we suggest that a group theoretic framework can unify ICM-based approaches and provide useful tools to study generative models in general. This involves defining a group of generic transformations that perturb the relationship between mechanisms and causes, as well as an appropriate contrast function to assess the genericity of the cause-mechanism relationship. We show that this framework encompasses previous ICM approaches \citep{icml2010_062,shajarisale2015}.

 In addition, while previous methods based on ICM where focused on cause-effect pairs (where we need to infer which of the variables is the cause and which is the effect), the present paper shows that the group theoretic view provides a  framework to study causal generative models in a more general setting that includes latent variable models. 
 
In \cref{sec:occlusion}, we introduce our framework with an informal example in visual inference. \Cref{sec:framework} presents the framework in full generality. Then \cref{sec:gtview} shows how previous causal inference approaches fit in this framework. Finally \cref{sec:groupnmf} shows how it can be applied to analyze and improve unsupervised learning algorithms. Proofs are provided in appendix.

%% file: occlusionarxivv1.tex
\section{An example in visual inference}\label{sec:occlusion}
\subsection{Occlusion and illusory contours}
In this section, we illustrate fundamental properties of the group theoretic approach by studying a simplified version of a basic inference problem for visual perception: the identification of partially occluded objects. In two dimensional naturalistic visual scenes, an object can partially mask other objects standing behind it in the scene. This phenomenon is usually well identified by the human visual system, but remains a major challenge for robust object detection in computer vision. Interestingly, even human perception of occlusion can be misled in specific examples of illusory contours. This is the case of the well know Kanizsa's triangle shown in \cref{fig:kaniza}. What is important for the illusory contour to emerge is the precise alignment between the edges of the Pac-Man-shaped inducers, instigating the completion of each aligned segment pair into a larger edge. 
One way to describe the specificity of such figure is to count the number of lines carrying the straight edges of the three Pac-Man shapes in the figure: there are only three lines, which is atypically (or suspiciously) small for a figure made of three objects totaling six straight edges. The idea that a configuration is "atypical" lies at the heart of our causal inference framework, and we can indeed formulate scene understanding tasks related to object occlusion as causal inference problems.

\begin{figure*}
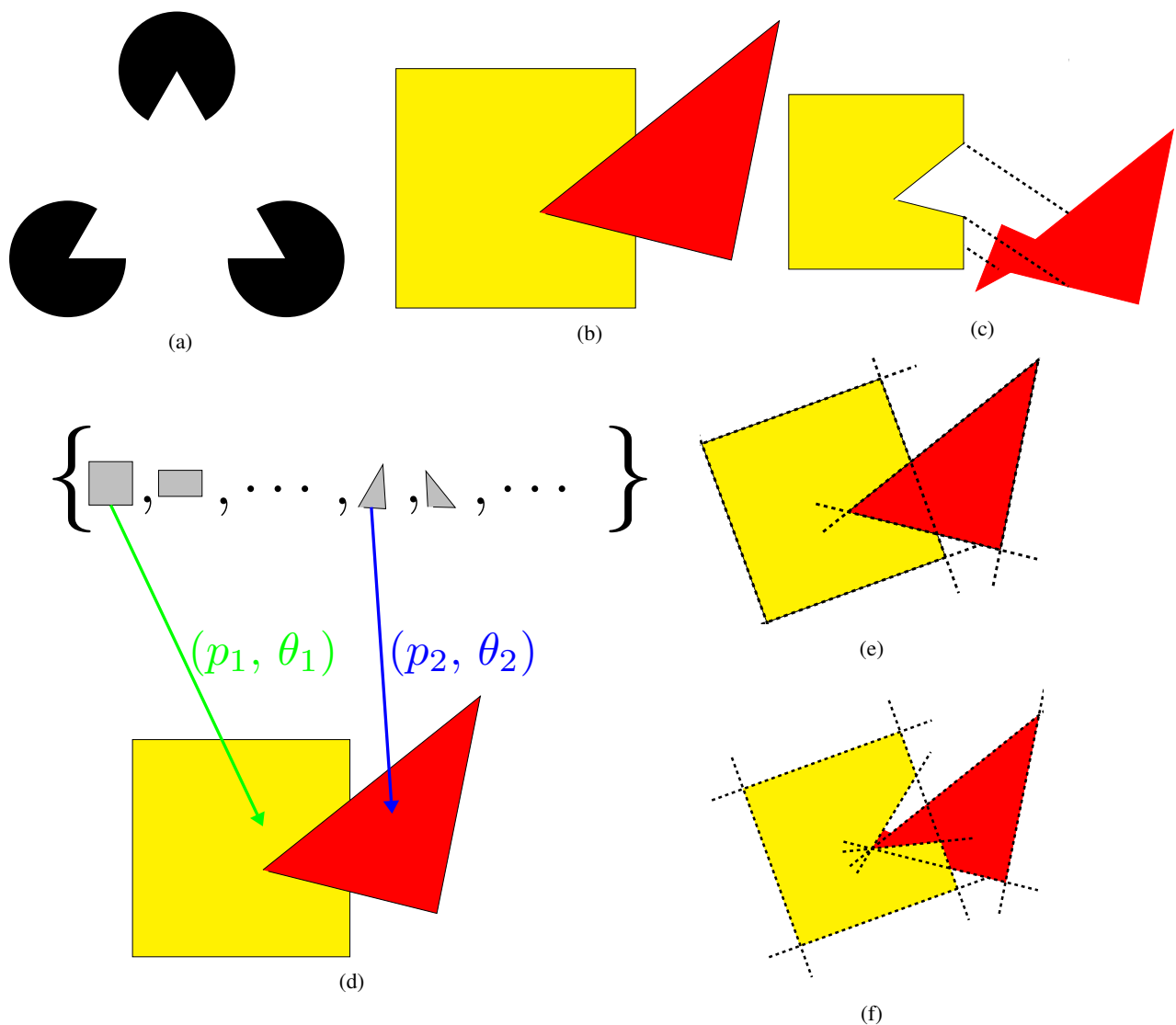

\begin{subfigure}{.3\textwidth}
\includestandalone[width=\textwidth]{figures/kaniza}
\subcaption{\label{fig:kaniza}}
\end{subfigure}
\hspace*{.5cm}
\begin{subfigure}{.34\textwidth}
\includestandalone[width=\textwidth]{figures/fig_occlusion1}
\subcaption{\label{fig:occlusion1}}
\end{subfigure}
\begin{subfigure}{.34\textwidth}
\includestandalone[width=\textwidth]{figures/fig_occlusion3}
\subcaption{\label{fig:occlusion3}}
\end{subfigure}\\
\begin{subfigure}{.6\textwidth}
\includestandalone[width=\textwidth]{figures/fig_occlusion2}
\subcaption{\label{fig:occlusion2}}
\end{subfigure}
\begin{minipage}{.3\textwidth}
\centering
\begin{subfigure}[t]{\textwidth}
\includestandalone[width=\textwidth]{figures/fig_occlusion1rand}
\subcaption{\label{fig:occlusion1rand}}
\end{subfigure}\\
\begin{subfigure}[t]{\textwidth}
\includestandalone[width=\textwidth]{figures/fig_occlusion3rand}
\subcaption{\label{fig:occlusion3rand}}
\end{subfigure}
\end{minipage}
\caption{(a) Kanizsa's triangle. (b) Scene of a yellow square occluded by a red triangle. (c) Example of objects leading to the same scene as (b) but where red occludes yellow (dashed segments link superimposed points). (d) Generative model for (b). (e) New scene resulting from an arbitrary rotation of the yellow object in (b). Dashed lines indicate straight lines carrying the edges of the objects.  (f) Same as (e) for the case presented in (b). See text for explanations.}
\vskip -.6cm
\end{figure*}

\subsection{Formulation of the causal inference problem}
We state the following scene understanding problem: two polygonal objects (with different colors) appear in a visual scene, occluding each other, and we want to infer which object partially occludes the other one. An example of such scene is represented on \cref{fig:occlusion1}. In this particular example, a likely interpretation of the scene is that the red triangle appears in front of a yellow square. However, one could imagine on the contrary that the yellow object is in the foreground and occludes the red object, by picking the objects shown in \cref{fig:occlusion3} for example. Such configuration is however intuitively unlikely if the precise positions and orientations of both objects are not ``tuned" to lead to the scene shown in \cref{fig:occlusion1}, that would give the ``illusion" that a yellow square is occluded. Such considerations have led vision scientists to formulate a \textit{generic viewpoint assumption} in order to perform inference \citep{freeman1994generic}. Such scene understanding problem can be considered as a causal inference problem, as it amounts to inferring a property of the generating mechanism (the objects and their configuration) that leads to an observation (the visual scene).

We now specify more precisely a generative model of the visual scene (see \cref{fig:occlusion2} for an illustration): from a large collection of polygons, a first object is selected and put in the scene at position $p_1$ with orientation $\theta_1$. Then a second object is selected from the collection and put on top of the first at position $p_2$ and orientation $\theta_2$. Under this generative model, both configurations (red in front or yellow in front) are possible. In order to determine which configuration is more likely, we resort to an additional postulate of independence of cause and mechanism. The cause being the pair of objects, and the mechanism being the set of positions and orientations they have in the scene, we assume that these last parameters are picked independently from the geometry of the chosen objects. As a consequence, if we apply a random rotation to one of the objects, we expect that for most cases, some global properties of the image will be preserved, such that the original scene can be qualified as ``typical''. In the following, we will see that the number of lines in the scene is a simple but very useful global property.

Indeed, if we apply now a random rotation to the yellow object under the hypothesis that the red triangle is the object in front, we obtain a modified scene in \cref{fig:occlusion1rand}, that is in some respect similar to the original figure. In particular, the total number of lines carrying objects' edges is 7 in both cases (4 for the square and 3 for the triangle). On the contrary, under the hypothesis that the yellow object is in front, we typically get a configuration like the one in \cref{fig:occlusion3rand}, which leads to a larger number of lines in the scene (at least 2 more, if we do not count the unknown hidden edges of the red object). To summarize this experiment, if we assume the red object is in front, the number of straight lines in the scene (7) is ``typical'' since an arbitrary rotation of one object will typically lead to the same number of edges. On the contrary, if we assume the yellow object is in front, the number of lines in the scene (still 7) is suspiciously low with respect to what it becomes when modifying the generative model with an arbitrary rotation applied to one object. To put it differently, under the ICM assumption that the orientations are chosen at random independently from the shape of objects put in the scene, configurations with a typical number of lines occur with high probability, such that it is much more likely that the red object lies in front.

\subsection{Group invariance view}
The above reasoning can be abstracted in order to better understand the critical elements that are used for causal inference. First, a generative model has been introduced: given two polygonal objects $O_1$ and $O_2$, a scene $S$ is generated according to the function
$$
S=m_{p_1,\theta_1,p_2,\theta_2}(O_1,O_2)
$$
where $m_{p_1,\theta_1,p_2,\theta_2}$ represents the mechanisms that puts first $O_1$ on the scene at position $p_1$ with orientation $\theta_1$, then puts $O_2$ at position $p_2$ with orientation $\theta_2$ in front of $O_1$.
We have seen that an important characteristic of the scene is the number of lines carrying edges. Lets call $C(S)$ this number for a given scene $S$. Then we inferred whether $C(S)$ is typical for the observed scene given an assumption about which object stands in front in the scene. Assume first the red object $R$ stands in front of the yellow one, the generative model is
\begin{equation}\label{eq:occlusion1}
S=m_{p_Y,\theta_Y,p_R,\theta_R}(Y,R).
\end{equation}
We test whether this model is typical by introducing a random rotation $r_\phi$ applied to the yellow object, such that we create a modified scene
$$
S_\phi=m_{p_Y,\theta_Y,p_R,\theta_R}(r_\phi Y,R)
$$
Then we say that \cref{eq:occlusion1} is typical if $C(S)=C(S_\phi)$ for most choices of $\phi$. This happens to be the case in our example, as we rotate the object around its center, which preserves the occlusion. As a consequence, being typical is reflect by the \textit{invariance} of the value $C$ with respect to a group of transformations: the rotations.

The alternative generative model (where the yellow object is in front) can be written as
\begin{equation}\label{eq:occlusion2}
S=m_{p_R,\theta_R,p_Y,\theta_Y}(R,Y).
\end{equation}
Such that we can check whether it is typical by comparing it to the rotated model
$$
S'_\phi=m_{p_R,\theta_R,p_Y,\theta_Y}(R,r_\phi Y)
$$
Then one can observe that for most choices of $\phi$\footnote{more precisely for almost all choices of $\phi$}, we have
$$
C(S'_\phi)\geq C(S)+2,
$$
and therefore the model described by \cref{eq:occlusion2} is atypical, as witnessed by the lack of invariance of $C$ values, and thus less likely than the model of  \cref{eq:occlusion1} to explain the scene.
The key elements that are used for causal inference are thus:
1/ a generative model of the observed data,
2/ a group of generic transformations (here rotations) that can be applied to the model to simulate "typical" configurations of the generative model and
3/ a contrast (here the number of lines in the scene) that can be evaluated on both the observed data and in typical configurations of the model. Invariance of the contrast to generic transformations then indicates observations are typical. All these elements will be used for defining a general framework for causal inference in section \ref{sec:framework}.

%% file: figures/kaniza.tex
\begin{tikzpicture}

\node[draw=none,minimum size=3cm,regular polygon,regular polygon sides=3] (a) {};

\foreach \x in {1,2,...,3}
  \fill (a.corner \x) circle[radius=.7cm];
\node[draw=none,fill=white,minimum size=3cm,regular polygon,regular polygon sides=3] (a) {};

\end{tikzpicture}

%% file: figures/fig_occlusion1.tex
\begin{tikzpicture}

\draw[fill=yellow]  (-3,3) rectangle (2,-2);
\draw[fill=red] (0,0) node (v1) {} -- (5,4) -- (4,-1) -- (v1) -- cycle;
\end{tikzpicture}

%% file: figures/fig_occlusion3.tex
\begin{tikzpicture}

\begin{scope}[shift={(1,1)}]
\draw[fill=yellow]  (-3,3) rectangle (2,-2);
\draw[fill=white] (0,0) node (v1) {} -- (5,4) -- (4,-1) -- (v1) -- cycle;
\draw[white,fill=white]  (2.05,5) rectangle (5,-3);
\end{scope}

\draw[red,fill=red,xshift=4cm,yshift=-1cm] (0,0) node (v1) {} -- (5,4) -- (4,-1) -- (v1) -- cycle;
\draw[red,fill=red,xshift=5.6cm,scale=.4,yshift=-1cm,rotate=170] (0,0) node (v1) {} -- (5,4) -- (4,-1) -- (v1) -- cycle;
\begin{scope}[shift={(1,1)}]
\clip  (2.05,5) rectangle (5,-3);
\draw[dashed,line width=2pt,shift={(-1,-1)}] (1,1) -- (4,-1);
\end{scope}

\draw[dashed,line width=2pt] (6,.6) -- (3,2.6);

\draw[dashed,line width=2pt] (3,0.5) -- (6,-1.5);

\end{tikzpicture}

%% file: figures/fig_occlusion2.tex
\begin{tikzpicture}
\begin{scope}[scale = 1]
	
	\draw[fill=yellow]  (-3,3) rectangle (2,-2);

	\draw[fill=red] (0,0) node (v1) {} -- (5,4) node (v2){} -- (4,-1) node (v3){} -- (v1) -- cycle;
\end{scope}

\begin{scope}[scale = .2,shift={(-10,44)}]
	
	\node[scale=7] at (-12,1) {$\{$};
	\draw[fill=lightgray]   (-10,3) node (elem1){} rectangle (-5,-2);
	\node[scale=4] at (-3.3,-2) {$,$};
	\draw[fill=lightgray]   (3,2) rectangle (-2,-1);
	\node[scale=4] at (12.3,-1) {$,\cdots,$};
	\draw[fill=lightgray,shift={(20,-2.5)},rotate=15] (1,0) node (elem2) {} -- (5,4) node (vred2){} -- (4,-1) node (vred3){} -- (elem2) -- cycle;
	\node[scale=4] at (27.3,-2) {$,$};
	\draw[fill=lightgray,shift={(28,-2.5)},rotate=17] (1,0) node (vred1) {} -- (2,4) node (vred2){} -- (4,-1) node (vred3){} -- (vred1) -- cycle;
	\node[scale=4] at (40.3,-1) {$,\cdots$};
	\node[scale=7] at (52,1) {$\}$};
\end{scope}

\draw[-Triangle,line width=2pt,blue]  ($(elem2)+(.3,0)$) -- ($.33*(v1)+.33*(v2)+.33*(v3)+(0,.3)$);
\node[scale=3,blue] at (4.6,5) {$(p_2,\,\theta_2)$};

\draw[-Triangle,line width=2pt,green]  ($(elem1)+(.5,-1)$) -- ($(0,1)$);
\node[scale=3,green] at (0,5) {$(p_1,\,\theta_1)$};

\end{tikzpicture}

%% file: figures/fig_occlusion1rand.tex
\begin{tikzpicture}

\draw[fill=yellow,rotate=20]  (-3,3) rectangle (2,-2);
\draw[fill=red] (0,0) node (v1) {} -- (5,4) -- (4,-1) -- (v1) -- cycle;
\draw[shorten <= -10cm,shorten >= -10cm,rotate=20,line width=2pt,dashed]  (-3,3) -- (-3,-2);
\draw[shorten <= -1cm,shorten >= -1cm,rotate=20,line width=2pt,dashed]  (-3,3) -- (2,3);
\draw[shorten <= -1cm,shorten >= -1cm,rotate=20,line width=2pt,dashed]  (2,3) -- (2,-2);
\draw[shorten <= -1cm,shorten >= -1cm,rotate=20,line width=2pt,dashed]  (-3,-2) -- (2,-2);
\draw[shorten <= -1cm,shorten >= -1cm,line width=2pt,dashed]  (v1) -- (5,4);
\draw[shorten <= -1cm,shorten >= -1cm,line width=2pt,dashed]  (v1) -- (4,-1);
\draw[shorten <= -1cm,shorten >= -1cm,line width=2pt,dashed]  (5,4) -- (4,-1);

\end{tikzpicture}

%% file: figures/fig_occlusion3rand.tex
\begin{tikzpicture}
\draw[red,fill=red,xshift=0cm] (0,0) node (v1) {} -- (5,4) -- (4,-1) -- (v1) -- cycle;

\draw[red,fill=red,xshift=1.6cm,scale=.4,rotate=170] (0,0) node (v1) {} -- (5,4) -- (4,-1) -- (v1) -- cycle;

\begin{scope}[rotate=20]
	\clip (0,0) -- (5,4) -- (-4,4) -- (-4,-3) -- (4,-3) -- (4,-1) -- (0,0) -- cycle;
	\draw[yellow,fill=yellow]  (-3,3) rectangle (2,-2);
\end{scope}

\node at (0,0) (v1){};
\draw[shorten <= -1cm,shorten >= -1cm,line width=2pt,dashed]  (v1) -- (5,4);
\draw[shorten <= -1cm,shorten >= -1cm,line width=2pt,dashed]  (v1) -- (4,-1);
\draw[shorten <= -1cm,shorten >= 3cm,rotate=20,line width=2pt,dashed]  (v1) -- (5,4);
\draw[shorten <= -1cm,shorten >= 1cm,rotate=20,line width=2pt,dashed]  (v1) -- (4,-1);
\draw[shorten <= -1cm,shorten >= -1cm,rotate=20,line width=2pt,dashed]  (-3,3) -- (-3,-2);
\draw[shorten <= -1cm,shorten >= -1cm,rotate=20,line width=2pt,dashed]  (-3,3) -- (2,3);
\draw[shorten <= -1cm,shorten >= -1cm,rotate=20,line width=2pt,dashed]  (2,3) -- (2,-2);
\draw[shorten <= -1cm,shorten >= -1cm,rotate=20,line width=2pt,dashed]  (-3,-2) -- (2,-2);

\draw[shorten <= -1cm,shorten >= -1cm,line width=2pt,dashed]  (5,4) -- (4,-1);
\end{tikzpicture}

%% file: frameworkarxivv1.tex
\section{Group theoretic framework}
\label{sec:framework}
\subsection{Background and related work}
Many machine learning approaches rely on statistical models in order to approximate observed data, ranging from simple linear regression models to the recently introduced Generative Adversarial Networks (GANs) \citep{goodfellow2014generative}. In order for these models to serve their purpose, they have to represent the observations as faithfully as possible. Such property can be evaluated in a purely statistical sense by testing whether the probability distribution of the model is as close as possible to the empirical distribution of the data (taking into account that such procedure should be properly designed to avoid overfitting). However, inferring a causal model goes beyond this statistical criterion by imposing that the fitted model should in some sense capture the structure of the true data generating process. Concepts pertaining to causality are well formulated using Structural Equation Models (SEMs) or Structural Causal Models (SCMs), which describe the relationship between different variables (observed or hidden) as a set of structural equations, each of them taking the form \footnote{in such equation, each right-hand-side variables may refer to either endogenous variables (i.e. a factor whose value is determine by other variables in the system under consideration) or exogenous variables (determined by factors outside of the system)}
\[
v_k \coloneqq f_k(v_1,\cdots,v_{n}).
\]  
Such equations represent more than algebraic dependencies between the variables, as indicated by the asymmetry of the ``colon-equals" symbol, which suggests that the left-hand side variable is in some way defined based on the right-hand side variables. Broadly construed, it means that this relation would still hold if an external agent were to intervene on one or several right-hand side variables (the so called do-operator), or alternatively that we can formulate counterfactuals : ``what would have happened if one right-hand side variable had been different" (see \citet{pearl2000causality} for an overview). As a consequence, a properly inferred causal generative model based on such structural equation offers more robustness to changes in the environment than are purely statistical models. This includes important cases related to transfer leaning such as covariate-shift or changes in the input distribution \citep{zhang2013domain,zhang2015multi,peters2016causal,bareinboim2016causal,rojas2015causal}, making causal generative models highly relevant in machine learning. In addition, a correct causal generative model offers a formalism to describe and understand the actual mechanisms underlying the observed data and is thus a central goal of experimental sciences.

From \cref{sec:occlusion}, we can deduce that our framework virtually probes structural equations by a counterfactual reasoning which could be stated as: ``what would happen if we were to apply a generic transformation to a given variable or mechanism of the SCM". This virtual intervention is represented in \cref{fig:genericintervention} for a SCM with a single cause and mechanism. The applied transformation is generic in the sense that it is sampled at random from a large set of transformations that turns a variable/mechanism in another one that is as likely to occur in a similar SCM. The outcome of this virtual intervention is tested by quantifying whether the counterfactual outcome is qualitatively different from the observed outcome for \textit{most} generic transformations. If it is, we say that our observations would be atypical for the considered causal model, such that this is unlikely for the causal generative model to have generated such observations. This is justified by an ICM assumption: in order to have generated the particular observation, it is unlikely that this precise mechanism would be selected by Nature independently from its input; in the contrary, the mechanism would likely need to be chosen specifically for this input to generate the observed outcome.

In our framework, the set of generic transformations is chosen to be a (compact) group. We will refer to the appendix for the relevant definitions regarding group theory. The readers may however just assume the compact group is a set of invertible transformations applied to causes and equipped with a "uniform" probability measure. The choice of this particular structure is motivated by the fact that group actions combine well with a general structural equation framework. Although extension to other structures (such as semigroups) may prove useful, we will see that the group setting can describe a large variety of causal generative models.
\michel{note that conjugacy is also a left group action on the mechanism, no need to consider it separately?}
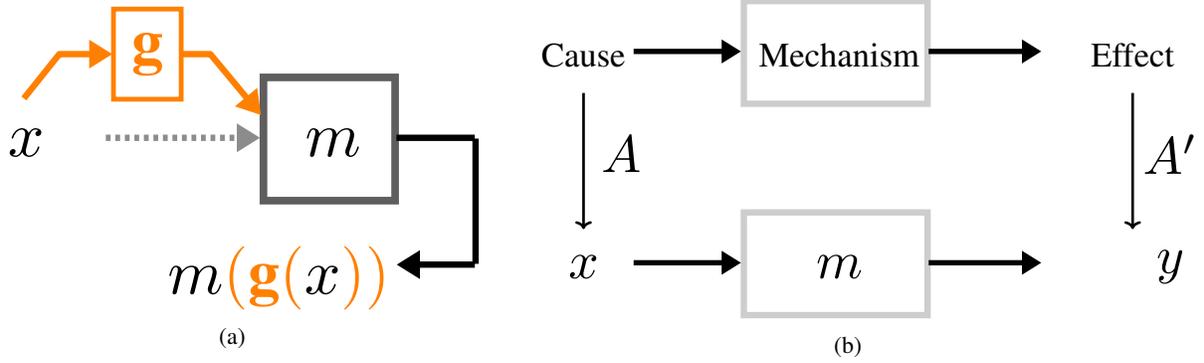
\begin{figure*}
	\begin{subfigure}{.4\textwidth}
		\centering
		\scalebox{.6}{\input{figures/group_generic_fig1bis.tikz}}
		\subcaption{\label{fig:genericintervention}}	
	\end{subfigure}\hspace*{.5cm}
	\begin{subfigure}{.52\textwidth}
		\centering
		\scalebox{.5}{\input{figures/group_princip_fig2bis.tikz}}
		\subcaption{\label{fig:attribute}}
	\end{subfigure}	
	\caption{(a) Principle of the group theoretic framework: a generic transformation is introduced between the cause and the mechanism. (b) Introduction of the concept of attribute to describe a structural causal model.}
	\vskip -.5cm
\end{figure*}
\subsection{Formal definition}
We state our framework first in full generality: we assume plausible causes and effects do not need to belong to the same sets or to be the same type of object. Importantly, the causes and effects may be deterministic or random. In addition, we do not assume invertibility of the cause effect relationship. To allow a quantitative analysis of the system, it is necessary to characterize SCM variables by a mathematical object that we call their attribute. The covariance matrix is an example of attribute for a multivariate random variable, and typically the attribute of a random variable will be a function of its probability distribution. More generally, any function of a cause or effect can potentially be considered as an attribute. Speaking formally, given an effect generated by a cause through the mechanism as described in \cref{fig:attribute}, we measure attributes of cause and effect using functions $A$ and $A'$ with codomain $\A$ and $\A'$ respectively. To allow a less formal presentation, we will abusively consider the mechanism $m$ as acting directly on the attribute space $\A$, and $x$ and $y$ will indicate indistinctly the cause and effect or their attribute.

Applying the ICM framework requires assessing genericity of the relationship between input and mechanism quantitatively. For that we define two objects:
\begin{itemize}
	\item the \textit{generic group} $\G$ is a compact topological group that acts on $\A$, thus equipped with a unique Haar probability measure $\mu_\G$ (see appendix for group theoretic definitions),
	\item the \textit{Cause-Mechanism Contrast}\footnote{we use the term contrast in reference to Independent Component Analysis, where contrast functions are also used to measure independence} (CMC) $C$ is a real valued function\footnote{in this short paper we do not elaborate on the requirements regarding measurability of the functions to keep the presentation readable} with domain $\A'$.
\end{itemize}

The CMC and generic group introduced in such a way allow to compute the expected value when randomly "breaking" the cause-mechanism relationship using generic transformations.
\begin{defn}
Given a CMC $C$, the \textit{Expected Generic Contrast} (EGC) of a cause mechanism pair $(x,m)$ is defined as:
\begin{equation}\label{eq:EGC}
\textstyle{\langle C\rangle_{m,x}=\mathbb{E}_{g\sim \mu_\G}C(mgx) = \int_{\G}C(mgx)\,d\mu_\G(g)}\,.
\end{equation}
We say that the relation between $m$ and $x$ is $\G$-generic under $C$ (or the pair $(x,m)$ is $\G$-generic), whenever
\begin{equation}\label{eq:genericity}
\textstyle{C(mx)=\langle C\rangle_{m,x}}
\end{equation}
holds approximately.
\end{defn}
We call ~(\ref{eq:genericity}) the \textit{genericity equation}. Note that this equation is used to express an idealized ICM postulate (hence the term ``holds approximately") that is not meant to be satisfied exactly in practice but justified by assuming that in appropriate contexts, the generic distribution concentrates around its mean (see \citep{icml2010_062} for an example). Such concentration of measure results are hard to obtain in general, and we will leave this to further work to focus on the properties of the genericity equation in this paper. Genericity can be formulated more rigorously as a statistical test assessing whether $C(mx)$ is likely to be sampled from the null hypothesis whose distribution is generated by $C(mgx),\,g\sim\mu_\G$ \citep{frW_UAI}. 

\subsection{Invariant generative models}\label{sec:invagen}
There is an interesting probabilistic interpretation of the concept of genericity. If we are given a generative model such that the cause $x$ is a single sample drawn from a distribution $\mathcal{P}_\mathcal{X}$ (see \cref{fig:genmodsamp}).  To estimate genericity irrespective of the possible values of $x$ we consider the \textit{generic ratio} $C(mx)/\langle C\rangle_{m,x}$: this quantity should be close to one with a large probability in order to satisfy ICM assumptions and in several cases (as in the example of  \cref{sec:gtview}) leads to a scale invariant measure of genericity. Assume $\mathcal{P}_\mathcal{X}$ is a $\G$-invariant distribution, under mild assumptions \citep{wijsman1967cross} it can be parametrized as $x=g\tilde{x}$ for $\tilde{x}$ and $g$ independent RVs and $g\sim\mu_{\G}$. Then
\begin{equation}\label{eq:ratio}
\mathbb{E}_x \!\!\left[\frac{C(mx)}{\langle C\rangle_{m,x}}\right]\!=\!\mathbb{E}_{\tilde{x}}\mathbb{E}_{\tilde{g}}\!\! \left[\frac{C(mg\tilde{x})}{\langle C\rangle_{m,g\tilde{x}}}\right]\!=\!1
\end{equation} 
This tells us that the postulate of genericity is true at least ``on average" for the generative model. On the contrary, if this average would be different from 1 as it may happen for non-invariant distributions, the method is unlikely to succeed. As represented on \cref{fig:genmodsamp}, the same reasoning can be applied when sampling the mechanism from an invariant distribution. Such reasoning can be generalized to more complex causal structures involving multiple mechanisms and variables. Note also that some variables may be unobserved (latent), as this is the case for many generative models in machine learning (see \cref{fig:latent}). This case will be investigated in \cref{sec:groupnmf}. 
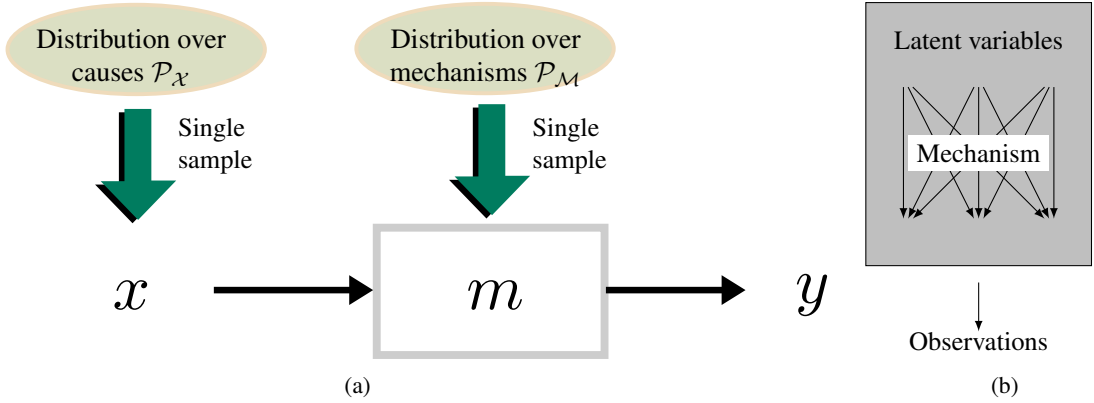
\begin{figure*}
\begin{subfigure}{.6\textwidth}
\scalebox{.62}{\input{figures/genermodel_fig1bis.tikz}}
\subcaption{\label{fig:genmodsamp}}
\end{subfigure}
\hspace*{1cm}
\begin{subfigure}{.3\textwidth}
	\centering
	\scalebox{1}{\input{figures/gen_model.tex}}
	\subcaption{\label{fig:latent}}
	\end{subfigure}
\caption{(a) Generative model including distributions over causes and mechanisms. (b) Causal structure of a latent generative model.}
\vskip -.5cm
\end{figure*}

\begin{commenth}
Let us now mention some interpretations about \domi{this sentence with interpretations still sounds strange} the generative model using this kind of normalized contrast (without loss of generality). Independence of cause and mechanism can be viewed as relying on the  assumption of a generative model from which the observed causes and/or mechanisms are sampled. If we assume either the cause or mechanism distribution are $\G$-invariant, it may be possible (relying on additional assumptions: see \citep{Eaton1989,wijsman1967cross,farrell2012multivariate,barndorff2012decomposition}) to represent the  $\G$-invariant random variable as a product of two independent variables: a cross-section variable $Y$ and a Haar-distributed random element from $\G$. Under these hypotheses, the EGC of $\G$-separable CMC appears naturally in the population mean. 
\domi{the reader gets no idea about what these assumptions are. Moreover, the sentence inspires the
reviewers to ask to what extent these references already contain the essential idea...}

\begin{prop}\label{prop:invar_gener}
Assume the mechanism $m$ is fixed and the cause attribute $X$ is randomly drawn from a  $\G$-invariant generative model and meet the conditions to be represented as $X=gY$ where $g\sim \mu_\G$ and $Y\sim p_Y$ are independent random variables. 
\domi{This assumption sounds strange and cumbersome, why not simply saying:"...and assume that $X=gY$ where  $g$ is randomly drawn from the Haar measure...and $Y$ is fixed"? What's the point with randomizing $Y$ at all?}
Then for any multiplicatively separable CMC
$$
\mathbb{E}_{X|Y=y_0} \frac{C(mX)}{\langle C\rangle_{m,X}}=1.
$$  
The same results holds trivially by replacing conditional expectations by full expectations. 
\end{prop}
\begin{proof}
We assume the contrast is normalized (without loss of generality). Then for the first case
$$
\mathbb{E}_{X|Y=y_0} \frac{C(mX)}{\langle C\rangle_{m,X}}=
\mathbb{E}_{g\sim \mu_\G} \frac{C(mgy_0)}{C(mu)C(gy_0)}= \frac{\mathbb{E}_{g\sim \mu_\G}C(mgy_0)}{C(mu)C(y_0)}.
$$  
\end{proof}
A similar result is again true whenever the mechanism is $\G$-invariantly distributed (possibly together with an independent $\G$-invariant distribution of the cause).
These results show that the \emph{genericity ratio} between the contrast and the EGC  is 1 on average, if we sample from many systems distributed according to an $\G$-invariant generic model based on the ICM postulate. Whenever for the model at hand a concentration result guaranties that this ratio is close to its population mean with high probability for one single observation, we can use the closeness of this ratio to 1 to test whether the ICM postulate holds for the causal model and choose the plausible causal direction accordingly. Most ICM-based techniques follow this principle, however note that if the contrast turns to be additively separable, the ratio turns into a difference. \domi{If we have introduced the hypothesis test earlier, we can now refer to it}

\subsection{Identifiability}\label{sec:indent}

Identifiability is a key issue for ICM based causal discovery. While answering this question is as far as we know dependent on the particular CMU at hand, we provide here a few definitions and results that hold for a general class of CMUs.
First, is is intuitive that if a mechanism turns an invariant attribute from $\I$ into another, it will be impossible to identify the causal direction with a CMC. As a consequence, the subset of  mechanisms that leave this set unchanged is named \textit{trivial}. 
\begin{defn}
For a given CMU $\Upsilon=(A,\M,\G)$, the subgroup of mechanisms that leave $\I$ unchanged is call the  \textit{trivial} subgroup
\[\T=\{m\in \M\colon m\mathcal{I}=\mathcal{I}\}.\]
\end{defn}
 Intuitively speaking $\T$ consists of all elements in $\M$ which transform one invariant attribute to another one.  Causal inference will be blind to such mechanism, hence the name trivial (see~\ref{sec:indent}). This intuition is made precise in the following proposition.\michel{these intuitions are not intuitive for most causality people I would say, let's find a better way to explain this}

\begin{prop}
For any CMU $\Upsilon=(A,\M,\G)$, we have
\[\T = \{m\in \M \colon m\G =\G m\}.\]
\end{prop}
Elements of $\T$ are to some extent not affected by randomization with elements of the generic group, as right composition by a generic element amounts to left composition by another element, and this will affect the identifiability of the causal direction (see section~\ref{sec:indent}).

\begin{defn}
For a given CMU $\Upsilon$ and CMC $C$, the blind set is the set of models such that cause and mechanisms are $\G$-independent for both the forward and backward direction  
$$
\mathcal{B}(\Upsilon,C)=\{(x,m)\in\X\times\M,C(mx)=\langle C\rangle_{m,x}\text{ and } C(x)=\langle C\rangle_{\inv{m},mx}\}
$$
\end{defn}
The following properties are a direct consequence of this definition.
\begin{prop}
For a given CMU $\Upsilon$ and any CMC $C$, trivial mechanisms cannot be identified 
$$
\X\times\mathcal{T} \subset \mathcal{B}(\Upsilon,C)
$$
\end{prop} 
\begin{proof}
Let $(x,m)\in \X\times \T$. The proof follows by unraveling the definitions:
\[\langle C\rangle_{m,x} = 
\int_G C(mgx)\, d\mu_\G(g) \stackrel{\tiny(\star)}{=} 
\int_G C(g'mx)\, d\mu_\G(g) = 
C(mx)\]
for some $g'\in \G$. For $(\star)$ we used the very definition of $\T$ ($m\G = \G m$ for all $m\in \M$). Since contrast $C$ is $\G$-invariant the expression in the integral does not depend on $g$ and the last equality follows. Similarly
\[\langle C\rangle_{m^{-1},mx} = 
\int_\G C(m^{-1} g mx)\,d\mu_\G(g) \stackrel{\tiny(\star)}{=}
\int_\G C(g'' m^{-1} mx)\,d\mu_\G(g) = 
\int_\G C(x)\, d\mu_\G(g) = C(x).
\]
\end{proof}
An intuitive case of this result is that generic mechanisms cannot be identified.
The trivial subgroup thus determines a minimal set of mechanisms that are not identifiable, whatever the contrast is chosen. As a consequence, the particular attention should be paid to the choice of the the mechanism and generic subgroups in order not to restrict the set of identifiable mechanisms.

\subsection{Concentration of measure in CMUs}
A important element for identifiability in ICM-based models is to have a concentration of measure property, such that, typically for large dimensions of the input space, an appropriate generative model results in typical CMC values that concentrate around a population mean. Here  we provide a general formalism for that in a group theoretic setting and discuss implications for generative models under ICM.

\begin{defn}
	A CMC $C$ is said to be uniformly bounded if there exists $f:\mathbb{R}^+\rightarrow [0,\,1]$ such that for all $(a,m)\in(\A\times\M)$
	$$
	\operatorname{P}_{g\sim \mu_\G}\left(\left|\frac{C(m g a)}{\mathbb{E}_{g\sim \mu_\G}C(m g a)}-1\right|>\epsilon\right)\leq f(\epsilon).
	$$
\end{defn}
Note that the upper bound does not depend on $m$ and $a$, hence the term \textit{uniformly}. 
Assuming $f(\epsilon)$ stays reasonably small for small $\epsilon$ uniformly boundedness implies that the Randomized Contrast Distribution $\mathcal{D}_{x,m}$ defined above concentrates around its mean, the ERC, independently from the values of $m$ and $x$. 
This concentration property is very useful to prove that the ICM holds approximately for a generative model with group invariance properties:
\begin{thm}\label{thm:gener_invariance}
	Let $\Upsilon$ be a CMU with CMC function $C$ uniformly bounded by $f$. Assume the cause $X$ is drawn at random from a $\G$-invariant distribution, then 
	\[
	\operatorname{P}_{X\sim P_X}\left(\left|\frac{C(m A(X))}{\langle C\rangle_{m,x}}-1\right|>\epsilon\right)\leq f(\epsilon)
	\]
\end{thm}
\begin{proof}\marek{no idea what's happening here...}
	By $\G$-invariance of $P_X$ we have:
	\[
	\operatorname{P}_{X\sim P_X}\left(\left|\frac{C(m A(X))}{\langle C\rangle_{m,x}}-1\right|>\epsilon\right)= 
	\operatorname{P}_{X\sim P_X,g\sim\mu_\G}\left(\left|\frac{C(m g A(X))}{\mathbb{E}_{g\sim 
			\mu_\G}C(m g A(X))}-1\right|>\epsilon\right)
	\]
	such that
	\begin{multline*}
	\operatorname{P}_{X\sim P_X}\left(\left|\frac{C(m A(X))}{\langle C\rangle_{m,x}}-1\right|>\epsilon\right)=
	\int_{X\in \X,g\in\G}\mathbbm{1}_{\left|\frac{C(m g A(X))}{\mathbb{E}_{g\sim \mu_\G}C(m g A(X))}-1\right|>\epsilon}dP(X)d\mu_\G\\
	\leq\int_{X\in \X}f(\epsilon)dP_X\leq f(\epsilon)
	\end{multline*}
\end{proof}

From this theorem, we can see that provided the randomized contrast concentrates uniformly around its mean, then a generative model with a group invariant probability distribution for the cause will approximately satisfy the independence criterion. The choice of the generic transformation group $\G$ to consider is thus tightly related to the type of invariance we assume for the generative model.

\subsection{Separable CMC}
Another key element for the identifiability of ICM based method is the forward-backward inequality, which guaranties that if there is independence in the forward direction, this independence is violated for the backward model. Due to its definition, some special invariance structure is expected from the ERC.
\begin{prop}
	For any CMC $C$, we have the following invariance properties. 
	For all $m\in\M$ and $x\in\X$
	$$
	\forall g,k\in\G, \langle C \rangle_{m,x}=\langle C \rangle_{g*m*g^{-1},k.x}
	$$
\end{prop}
\michel{Note that here we exploit the conjugacy stability. In case we enforce the stronger condition that $\M$ is stable by left and right $\G$ multiplication, the invariance would take a more general form, without $g^{-1}$. Moreover, for such results, check "invariance theory", and "reynolds operator"}

This invariance structure can lead, at least for some special cases of CMC, to a simple form of the ERC where the dependency on $m$ and $x$ can be separated. We show here sufficient properties of two types of such "separable" CMC that enable to derive such an inequality.
\begin{defn}
	A CMC is said to be multiplicatively $\G$-separable if it exists a function $K$ such that for all $m\in\M$ and $x\in\X$
	$$
	\langle C\rangle_{m,x}=K(m)C(A(x))
	$$
\end{defn}
Note that from the definition and group invariance of the CMC, we get immediately that $K(e)=1$. Moreover, using basic properties of invariant attributes and of the Haar measure, it is easy to derive the following elementary property.
\begin{prop}\label{prop:K-multiplicative}
	For a mutiplicatively separable CMC of a CMU with reference attribute $u$, we have the following properties of $K$ for all $m\in\M$
	$$
	K(m)=C(mu)/C(u),
	$$
\end{prop}
leading to the equation
\begin{equation}\label{prop:K-multiplicative}
\langle C\rangle_{m,x}=\frac{C(mu)}{C(u)}C(A(x))
\end{equation}
We can thus get a normalized contrast $C'$ by dividing by $C(u)$, for which we get 
\begin{equation}\label{prop:K-multiplicative-norm}
\langle C'\rangle_{m,x}=C'(mu)C'(A(x))
\end{equation}
We observed that according to proposition~\ref{prop:tracemultsep}, the trace is paradigmatic of a multiplicatively separable contrast, and the normalized trace $\tau_n$ is its corresponding normalized contrast. Interesting interpretations can be drawn from equation~\ref{prop:K-multiplicative-norm}: for a fixed invariant distribution $u$ we can define a map from mechanisms to attributes $e_u: m\mapsto  mu$ simply by looking at how mechanisms transform the $\G$-invariant attribute $u$. Using this mapping, the contrast under the genericity between mechanism $m$ and attribute $a$ factorizes exactly like an expectation under independence assumption
$$
C'(ma)=C'(\aleph_u(m))C'(a).
$$
Let us now draw some interpretations about the generative model using such normalized contrast (without loss of generality). Independence of cause and mechanism can be viewed as relying on the  assumption of a generative model from which the observed causes and/or mechanisms are sampled. If we assume either the cause or mechanism distribution are $\G$-invariant, it may be possible (relying on additional assumptions: see \citep{Eaton1989,wijsman1967cross,farrell2012multivariate,barndorff2012decomposition}) to represent the  $\G$-invariant random variable as a product of two independent variables: a cross-section variable $Y$ and a Haar-distributed random element from $\G$. Under these hypotheses, the EGC of $\G$-separable CMC appears naturally in the population mean. 

\begin{prop}\label{prop:invar_gener}
	Assume the mechanism $m$ is fixed and the cause attribute $X$ is randomly drawn from a  $\G$-invariant generative model and meet the conditions to be represented as $X=gY$ where $g\sim \mu_\G$ and $Y\sim p_Y$ are independent random variables. Then for any multiplicatively separable CMC
	$$
	\mathbb{E}_{X|Y=y_0} \frac{C(mX)}{\langle C\rangle_{m,X}}=1.
	$$  
	In the same way, if the cause $x$ is fixed and the mechanism $M$ is drawn from a $\G$-invariant generative model and meet the conditions to be represented as $M=gNg^{-1}$, then 
	$$
	\mathbb{E}_{M|N=n_0} \frac{C(Mx)}{\langle C\rangle_{M,x}}=1.
	$$
	The same results holds trivially by replacing conditional expectations by full expectations. A similar result is again true whenever mechanism and cause are independent and both $\G$-invariantly distributed.
\end{prop}
\begin{proof}
	We assume the contrast is normalized (without loss of generality). Then for the first case
	$$
	\mathbb{E}_{X|Y=y_0} \frac{C(mX)}{\langle C\rangle_{m,X}}=
	\mathbb{E}_{g\sim \mu_\G} \frac{C(mgy_0)}{C(mu)C(gy_0)}= \frac{\mathbb{E}_{g\sim \mu_\G}C(mgy_0)}{C(mu)C(y_0)}.
	$$  
	The second case is similar.
\end{proof}
This results shows that the \emph{genericity ratio} between the contrast and the EGC is  is on average 1, if we sample from many systems distributed according to an $\G$-invariant generic model based on the ICM postulate. Whenever for the model at hand a concentration result guaranties that this ratio is close to its population mean  with high probability for one single observation, we can use the closeness of this ratio to 1 to test whether the ICM postulate holds for the causal model and choose the plausible causal direction accordingly. Most ICM-based techniques follow this principle, however note that if the contrast truns to be additively separable, the ratio turns into a difference.

A multiplicative contrast can be obtained in the context of $C^*$ algebra, generalizing the case of the Trace Method. Consider an Hilbert space of operators (equiped with a scalar product $\langle .,.\rangle$, and assume attributes can be represented as real valued positive definite operators, while the action of a mechanism $m$ on attribute $a$ is done with an (possibly complex valued) operator $M$ through the equation
$$
ma=MAM^*
$$
where $*$  indicates the complex conjugate.
By choosing the positive linear functional $C(a)=\left\langle A,  e\right\rangle$ as a contrast ($e$ indicates the identity operator), we conjecture that is multiplicatively separable in a quite general case.
\begin{conj}
	The positive linear functional contrast is multiplicatively separable for are all positive definite attributes $A$ whenever the unitary operator $U$ diagonalizing $A$ "belongs" to the generic subgroup of operators $\G$ and
	$$
	\langle C\rangle_{m,x}=C(M M^*)C(A)
	$$
\end{conj}
\begin{proof}
	Starting with the expression of the contrast
	$$
	\mathbb{E}C(mgx)=\mathbb{E}\left\langle MGAG^*M^*,  e\right\rangle
	$$ 
	by bilinearity or sequilinearity we get
	$$
	\mathbb{E}C(mgx)=\left\langle \mathbb{E} \left[GAG^*\right]M^*,  M^*\right\rangle
	$$ 
	Using a variant of the spectral theorem \citep{hall2013quantum}, we may be able to show that $A$ is "diagonal" in the basis defined by some unitary operator $U$. Whenever $U$ belongs to the group $\G$, we conjecture that
	$$
	\left[GAG^*\right]=C(A).e
	$$
	and thus
	$$
	\mathbb{E}C(mgx)=C(A)\left\langle M^*e,  M^*\right\rangle=C(M M^*)C(A)
	$$
\end{proof}

Having such a form for the ERC enables to derive a forward backward inequality using one single additional assumption.
\begin{thm}
	Assume a multiplicatively separable CMC with reference attribute $u$, such that for all $m \in \M\setminus \{\T\} $, 
	\begin{equation}\label{assum:CMCineq}
	C(mu)C(m^{-1}u)>C(u)^2,
	\end{equation}
	then we have for all $m \in \M\setminus \{\T\} $, and all causal model $y:=mx$
	$$
	\frac{C(A(y))}{\langle C\rangle_{m,x}}\cdot\frac{C(A(x))}{\langle C\rangle_{m^{-1},y}}<1
	$$
	and if $m$ and $x$ are $\G$-independent then $y$ and $m^{-1}$ are not.
\end{thm}
\michel{Marek mentioned in the last meeting that the maximality assumption is important here. Check in which sense.}
\michel{Also, assumption \ref{assum:CMCineq} looks like a convexity property around the identity. Is it more general (maybe we can make the geometric mean appear...) we should be able to show it is a directional second order derivative that is strictly positive}
\begin{proof}
	By separability
	$$
	\frac{C(A(Y))}{\langle C\rangle_{m,x}}\cdot\frac{C(A(x))}{\langle C\rangle_{m^{-1},y}}=\frac{1}{K(m)K(m^{-1})}.
	$$
	By Proposition \ref{prop:K-multiplicative} and assumption~\ref{assum:CMCineq}
	$$
	K(m)K(m^{-1})=\frac{C(mu)C(m^{-1}u)}{C(u)^2}>1,
	$$
	leads immediately to the conclusion.
\end{proof}
Such an inequality is called forward-backward inequality, since it bounds the product of two ratios with the same form, one for the forward model and one from the backward model. Combined with the previously introduced result regarding the concentration of the forward CMC around its mean, it allows identifiability by ensuring that if ICM is ensured in the forward direction, it cannot be achieved in the backward direction.

\begin{example}
	The trace contrast satisfies
	$$
	\mathbb{E}_{U\sim \mu_{O(n)}}\left[\tr (\Sigma_{MUX})\right]=\textstyle{\frac{1}{n}}\tr (MM^T)\tr (\Sigma_X)
	$$
	and is thus multiplicatively separable with $K(M)=\tr (MM^T)/n$.
	In addition, condition~\ref{assum:CMCineq} required for the forward-backward inequality directly stems from the Cauchy-Schwartz inequality and states that, for all invertible matrices $M$ not co-linear to the identity
	$$
	\tr \left(MM^T\right)\tr \left(M^{-1}{M^T}^{-1}\right)<\left(\tr \left(MM^{-1}\right)\right)^2=\left(\tr \left({\bf I}_n\right)\right)^2=n^2
	$$
\end{example}
\michel{for extension to the case of quantum renyi entropies check "INTEGRATION WITH RESPECT TO THE HAAR MEASURE ON
	UNITARY, ORTHOGONAL AND SYMPLECTIC GROUP
	BENOIT COLLINS AND PIOTR SNIADY"
}

\michel{the contrast defined as the Harich-Chandra integral may be additively separable when we have a low rank input according to "A Fourier view on the R-transform and
	related asymptotics of spherical integrals
	A. Guionnet, M. Maıda", to double check though...}
However, CMC are not necessarily multiplicatively separable. Alternatively, it can happen that the ERC can be decompose in a sum of two terms associated to the cause and mechanism respectively. The above definition and results can easily be modified to accommodate this new structure.
\begin{defn}
	A CMC is said to be additively $\G$-separable if there exists a function $K$ such that for all $m\in\M$ and $x\in\X$
	$$
	\langle C\rangle_{m,x}=K(m)+C(A(x))
	$$
\end{defn}
Note that from the definition and group invariance of the CMC, we get immediately that $K(e)=0$. Moreover, using basic properties of invariant attributes and of the Haar measure, we can derive.
\michel{the familly of quantum entropies could be in principle additively separable, the free probability framework predicts that this is true asymptotically when the dimension of matrices goes to infinity, however it seems difficult to prove it for finite dimension. See "Polynomial convolutions and (finite) free probability, Adam W. Marcus" and "Finite Dimensional Statistical Inference,
	Oyvind Ryan" for attempts on free probability in the case of finite dimension}

\begin{prop}\label{prop:K}
	For an additively separable CMC of a CMU with reference attribute $u$, we have the following properties of $K$ for all $m\in\M$
	$$
	K(m)=C(mu)-C(u)
	$$
\end{prop}
Having such a form for the ERC enables to derive a new forward backward inequality.
\begin{thm}
	Assume an additively separable CMC with reference attribute $u$, such that for all $m \in \M\setminus \{\T\} $, 
	\begin{equation}\label{assum:CMCineq_add}
	C(mu)+C(m^{-1}u)>2 C(u),
	\end{equation}
	then we have for all $m \in \M\setminus \{\T\} $, and all causal model $y:=mx$
	$$
	C(A(y))-\langle C\rangle_{m,x}+C(A(x))-\langle C\rangle_{m^{-1},y}<0
	$$
	and if $m$ and $x$ are $\G$-independent then $y$ and $m^{-1}$ are not.
\end{thm}
\begin{proof}
	By separability
	$$
	C(A(y))-\langle C\rangle_{m,x}+C(A(x))-\langle C\rangle_{m^{-1},y}=-K(m)-K(m^{-1})
	$$
	By Proposition~\ref{prop:K} and assumption~\ref{assum:CMCineq}
	$$
	K(m)+K(m^{-1})=C(mu)+C(m^{-1}u)-2 C(u)>0,
	$$
	leads immediately to the conclusion.
\end{proof}

\begin{example}
	The negative entropy constrast is additively separable for a CMU of probability densities on the unit interval. Let $C(P_X)=-H(P_X)=\int_{[0,\,1]}P_X\log(P_X)$, using the generic group of translations modulo 1 and a differentiable bijective mechanism $f$ we get:
	$$
	\mathbb{E}_{g\sim \mu_\G}\left[ -H(f.g.P_X) \right]=-\int\log\left|\frac{df}{dx}(x)\right|dx-H(P_X).
	$$
	Hence the separability with $K(f)=-\int\log\left|\frac{df}{dx}(x)\right|dx$.
	For any $f$ different from the identity almost everywhere, the condition~\ref{assum:CMCineq_add} stems from the  strict positivity of the differential KL divergence\footnote{$D_{KL}(P\|Q)=\int log(dP/dQ) dP$} since:
	$$
	K(f)+K(f^{-1})=D_{KL}\left( \mathcal{U}\|f^{-1}.\mathcal{U}\right)+D_{KL}\left( \mathcal{U}\|f.\mathcal{U} \right)>0
	$$
	where $\mathcal{U}$ stands for the uniform distribution on the unit interval.
	\michel{Can we go further using information geometry? Prove some convexity leading to identifiability etc... in the case the attributes leave on a manifold (usually affine space...), the Lie group framework may help..}
\end{example}

\subsection{The case of linear CMC's}
{\color{red}
	When the contrast is linear we have  
	\[\int C=C \int mgx=C\int gx(m^{-1}(y))dy\]
}
\michel{what can we say when we can permute integral and contrast by linearity, and that m is regular enough such that we can change variables (diffeomorphisms)?}.

\end{commenth}

%% file: figures/group_generic_fig1bis.tikz
\definecolor{cc47800}{RGB}{196,120,0}
\definecolor{cc47801}{RGB}{70,70,70}
\definecolor{c788a21}{RGB}{120,138,33}
\definecolor{c000001}{RGB}{0,0,1}
\definecolor{c007c5f}{RGB}{0,124,95}
\definecolor{grayDelete}{RGB}{140,140,140}

\begin{tikzpicture}[y=0.80pt, x=0.80pt, yscale=-1.000000, xscale=1.000000, inner sep=0pt, outer sep=0pt]
\begin{scope}
	\path[draw=cc47801,opacity=0.870,miter limit=10.00,line
	    width=4.825pt,rounded corners=0.0000cm] (361.4134,218.9056) rectangle
	    (465.1547,316.0917);
	
	\node (A) at (236.5694,266.5391) {};
	\node (B) at (359.5711,266.5391) {};
	\draw[-Triangle,line width=4pt,dashed,color=grayDelete] (A) edge (B);
	\node (C) at (465.5694,266.5391) {};
	\node (D) at (528.5711,266.5391) {};
	\node (Cprime) at (465.5694,366.5391) {};
	\node (Dprime) at (528.5711,366.5391) {};
	\draw[-Triangle,line width=4pt] (C) -- (D) -- (Dprime) -- (Cprime);

	\node (E) at (296.5694,200.5391) {};
	\node (G) at ($(E)-(2.7cm,0cm)$) {};
	\node (H) at ($(E)-(3.5cm,-1cm)$) {};
	\node (I) at ($(E)-(1.5cm,0cm)$) {};
	\draw[-Triangle,line width=4pt,line join=round,color=orange] (H) -- (G.center) -- (I);
	\path[draw=orange,line width=3pt] ($(E)+(-1.5cm,1cm)$) rectangle  ($(E)+(0cm,-1cm)$) node[pos=.5,scale=4,color=orange]{$\textbf{g}$};
	\node (F) at (319.5711,200.5391) {};
	\draw[-Triangle,line width=4pt,line join=round,color=orange] (E) -- (F.center) -- ($(B)+(0cm,-.5cm)$);
	
	\node[align=center,scale=4](input) at (4.9cm,7.6cm) {$x$};
	
	\node[align=center,scale=4](mech) at (11.7cm,7.6cm) {$m$};
	\node[align=right,scale=4,orange] (output) at (10.5cm,10.6cm) {$\textcolor{black}{m}(\textbf{g}(\textcolor{black}{x}))$};

\end{scope}

%

\end{tikzpicture}

%% file: figures/group_princip_fig2bis.tikz
\definecolor{cc47800}{RGB}{196,120,0}
\definecolor{cc47801}{RGB}{70,70,70}
\definecolor{c788a21}{RGB}{120,138,33}
\definecolor{c000001}{RGB}{0,0,1}
\definecolor{c007c5f}{RGB}{0,124,95}

\begin{tikzpicture}[y=0.80pt, x=0.80pt, yscale=-1.000000, xscale=1.000000, inner sep=0pt, outer sep=0pt]
\begin{scope}
\path[draw=cc47801,opacity=0.270,miter limit=10.00,line
    width=4.825pt,rounded corners=0.0000cm] (361.4134,218.9056) rectangle
    (535.1547,316.0917);

\node (A) at (256.5694,266.5391) {};
\node (B) at (359.5711,266.5391) {};
\draw[-Triangle,line width=4pt] (A) edge (B);
\node (C) at (535.5694,266.5391) {};
\node (D) at (643.5711,266.5391) {};
\draw[-Triangle,line width=4pt] (C) edge (D);
\node[align=center,scale=4](input) at (5.9cm,7.6cm) {$x$};

\node[align=center,scale=4](mech) at (12.7cm,7.6cm) {$m$};
\node[align=left,scale=4] (output) at (21.5cm,7.6cm) {$y$};

\end{scope}
 \begin{scope}[shift={(0,-200)}]
 \path[draw=cc47801,opacity=0.270,miter limit=10.00,line
     width=4.825pt,rounded corners=0.0000cm] (361.4134,218.9056) rectangle
     (535.1547,316.0917);

 \node (A) at (256.5694,266.5391) {};
 \node (B) at (359.5711,266.5391) {};
 \draw[-Triangle,line width=4pt] (A) edge (B);
 \node (C) at (535.5694,266.5391) {};
 \node (D) at (643.5711,266.5391) {};
 \draw[-Triangle,line width=4pt] (C) edge (D);
\node[align=center,scale=2.6](inputOrg) at (5.9cm,7.6cm) {Cause};

\node[align=center,scale=2.6](mechOrg) at (12.7cm,7.6cm) {Mechanism};
\node[align=center,scale=2.6] (outputOrg) at (20.5cm,7.6cm) {Effect};

\end{scope}
 \draw[->,line width=2pt] ($(inputOrg)+(0cm,1cm)$) --($(input)+(0cm,-1cm)$);
 \draw[->,line width=2pt] ($(outputOrg)+(0cm,1cm)$) -- ($(output)+(-1cm,-1cm)$);
\node[align=center,scale=4](input) at (6.9cm,4.6cm) {$A$};

\node[align=center,scale=4] (output) at (21.5cm,4.6cm) {$A'$};

\end{tikzpicture}

%% file: figures/genermodel_fig1bis.tikz
\definecolor{cc47800}{RGB}{196,120,0}
\definecolor{cc47801}{RGB}{70,70,70}
\definecolor{c788a21}{RGB}{120,138,33}
\definecolor{c000001}{RGB}{0,0,1}
\definecolor{c007c5f}{RGB}{0,124,95}

\begin{tikzpicture}[y=0.80pt, x=0.80pt, yscale=-1.000000, xscale=1.000000, inner sep=0pt, outer sep=0pt]

  \path[draw=cc47801,opacity=0.270,miter limit=10.00,line
    width=4.825pt,rounded corners=0.0000cm] (361.4134,218.9056) rectangle
    (535.1547,316.0917);

\path[shift={(-28.71698,-32.01401)},draw=cc47800,fill=c788a21,opacity=0.270,miter
    limit=10.00,line width=2.400pt]
    (304.8291,114.9084)arc(-0.008:180.008:98.129906 and
    34.102)arc(-180.008:0.008:98.129906 and 34.102) -- cycle;

\path[shift={(237.32127,-32.70997)},draw=cc47800,fill=c788a21,opacity=0.270,miter
    limit=10.00,line width=2.400pt]
    (304.8291,114.9084)arc(-0.008:180.008:98.129906 and
    34.102)arc(-180.008:0.008:98.129906 and 34.102) -- cycle;

\begin{scope}[cm={{0.0,0.51476,-0.49753,0.0,(844.34353,-119.62392)}}]
    \begin{scope}[cm={{0.61811,0.0,0.0,0.61811,(459.80646,392.0967)}},fill=c000001]
      \path[fill=c000001] (31.9662,631.1048) -- (203.6921,631.1048) --
        (203.6921,569.4855) -- (298.1414,663.9348) -- (205.2683,756.8079) --
        (205.2683,697.7749) -- (31.9459,697.7749) -- (31.9662,631.1048) -- cycle;
    \end{scope}
    \begin{scope}[cm={{0.61811,0.0,0.0,0.61811,(454.80645,383.52527)}},fill=c007c5f]
      \path[fill=c007c5f] (31.9662,631.1048) -- (203.6921,631.1048) --
        (203.6921,569.4855) -- (298.1414,663.9348) -- (205.2683,756.8079) --
        (205.2683,697.7749) -- (31.9459,697.7749) -- (31.9662,631.1048) -- cycle;
    \end{scope}
  \end{scope}

\begin{scope}[shift={(20,0)},cm={{0.0,0.51476,-0.49753,0.0,(554.43134,-116.00195)}}]
    \begin{scope}[cm={{0.61811,0.0,0.0,0.61811,(459.80646,392.0967)}},fill=c000001]
      \path[fill=c000001] (31.9662,631.1048) -- (203.6921,631.1048) --
        (203.6921,569.4855) -- (298.1414,663.9348) -- (205.2683,756.8079) --
        (205.2683,697.7749) -- (31.9459,697.7749) -- (31.9662,631.1048) -- cycle;
    \end{scope}
    \begin{scope}[cm={{0.61811,0.0,0.0,0.61811,(454.80645,383.52527)}},fill=c007c5f]
      \path[fill=c007c5f] (31.9662,631.1048) -- (203.6921,631.1048) --
        (203.6921,569.4855) -- (298.1414,663.9348) -- (205.2683,756.8079) --
        (205.2683,697.7749) -- (31.9459,697.7749) -- (31.9662,631.1048) -- cycle;
    \end{scope}
  \end{scope}

\node (A) at (236.5694,266.5391) {};
\node (B) at (359.5711,266.5391) {};
\draw[-Triangle,line width=4pt] (A) edge (B);
\node (C) at (535.5694,266.5391) {};
\node (D) at (643.5711,266.5391) {};
\draw[-Triangle,line width=4pt] (C) edge (D);
\node[align=center,scale=1.7](title1) at (4.9cm,2.5cm) { Distribution over \\ causes $\mathcal{P_X}$};

\node[align=center,scale=1.7](title2) at (12.5cm,2.5cm) {Distribution over \\ mechanisms $\mathcal{P_M}$};

\node[align=center,scale=4](input) at (4.9cm,7.6cm) {$x$};
\node[align=left,scale=1.6](sampling1) at (6.7cm,4.4cm) {Single \\sample};

\node[align=left,scale=1.6](sampling2) at (14.3cm,4.4cm) {Single \\sample};

\node[align=center,scale=4](mech) at (12.7cm,7.6cm) {$m$};
\node[align=center,scale=4] (output) at (19.5cm,7.6cm) {$y$};

\end{tikzpicture}

%% file: figures/gen_model.tex
\begin{tikzpicture}

\node at (3,0.5) {Observations};
\draw[fill=lightgray]  (1.5,1.5) rectangle (4.5,5);
\node at (3,4.5) {Latent variables};
\node (v1) at (2,4) {};
\node (v3) at (3,4) {};
\node (v5) at (4,4) {};
\node (v6) at (4,2) {};
\node (v4) at (3,2) {};
\node (v2) at (2,2) {};
\draw[-latex]  (v1) edge (v2);
\draw[-latex]  (v3) edge (v4);
\draw[-latex]  (v5) edge (v6);
\draw[-latex]  (v5) edge (v4);
\draw[-latex]  (v5) edge (v2);
\draw[-latex]  (v1) edge (v6);
\draw[-latex]  (v1) edge (v4);
\draw[-latex]  (v3) edge (v2);
\draw[-latex]  (v3) edge (v6);
\node (v7) at (3,0.5) {};
\node (v8) at (3,1.4) {};
\draw[-latex]  (v8) edge (v7);
\node[fill=white] at (3,3) {Mechanism};
\node at (5,3) {};
\end{tikzpicture}

%% file: gt_viewICML17.tex
\section{A new view on existing methods}
\label{sec:gtview}
We show in this section that the group invariance framework encompasses previous causal inference methods that have been proposed in the literature to solve the pairwise case: given two observables $X$ and $Y$, can we decide between the alternatives ``$X$ causes $Y$" or ``$Y$ causes $X$"?

\subsection{The Trace Method}
We consider the case of $X$ and $Y$ $n$- and $m$-dimensional RVs, respectively, causally related by the linear structural equation 
\begin{equation}\label{eq:multivar_CM}
\textstyle{Y:=MX+E\,,}
\end{equation}
where $M$ is an $m\times n$ structure matrix and $E$ is a multivariate additive noise term. The Trace Method \citep{icml2010_062} postulates independence of cause and mechanism in this scenario as follows:
\begin{post}[Trace Condition]\label{pos:tc1} 
	The cause $X$ with covariance matrix $\Sigma_X$, and the mechanism with matrix representation $M$, are independent of each other if 
	\begin{equation}
	\textstyle{\tau_m(M\Sigma_X M^T) = \tau_n (\Sigma_X) \tau_m (MM^T) \,}
	\label{eq_tc}
	\end{equation}
	holds approximately, where $\tau_n(B)$ denotes the normalized trace ${\rm tr}(B)/n$. 
\end{post}
If we take the normalized trace as a contrast and use generic matrices $U$ distributed according to the Haar measure over the group $SO(n)$, $\mu_{SO(n)}$, the EGC writes
\[
\left\langle C\right\rangle_{M,X}=\mathbb{E}_{U\sim \mu_{SO(n)}}\tau_n ( M U\Sigma_{X} U^T M^T)\,.
\]
This quantity can be evaluated using the following result.
\begin{prop}\label{prop:tracemultsep}
	Let $U$ be a random matrix drawn from $SO(n)$ according to $\mu_{SO(n)}$ and let $A$ and $B$ be two symmetric matrices in $M_{n,n}(\mathbb{R})=\mathbb{R}^{n\times n}$. Then
	\begin{equation}\label{eq:expectU}
	\textstyle{\textstyle{\mathbb{E}_{U\sim\mu_{SO(n)}}\tr \left(U^T A U B\right)=\textstyle\frac{1}{n}\tr (A)\tr (B)}}.
	\end{equation}
\end{prop}
This leads to 
\[
\left\langle C\right\rangle_{M,X} = \tau_n (\Sigma_X) \tau_m (MM^T)+\tau_m(\Sigma_E)
\]
where $\Sigma_E$ denotes the additive noise covariance. As a consequence, the above Trace Condition (\ref{eq_tc}) is a genericity equation from a group theoretic perspective (note this condition remains unaffected by the additive noise term). 

While we will present elsewhere that many other causal inference approaches can be formulated with the present group theoretic framework, we provide in appendix the additional example of the Spectral Independence Criterion (SIC) approach \citet{shajarisale2015} in the context of time series.

%% file: groupNMFarxivv1.tex
\section{Application to unsupervised learning}
\label{sec:groupnmf}
This section develops the idea that unsupervised learning can be a new field of application for ICM principles based on group invariance.  Many unsupervised learning algorithms can indeed be thought of as generative models, and we propose to add a causal perspective to them in order to improve their characterization and inference.

\subsection{Causal generative models}
Classically generative models aim at modeling the probability distribution of observations. However, we often expect from such model to capture information about the true generative process, in order to better understand its underlying mechanisms. Take for example the case of clustering using Gaussian mixture models, when experimental scientists cluster a dataset, they expect that the resulting clusters reflect a reliable structure in their data. Put more explicitly, they expect that their clusters will be robust to moderate changes in the data generating mechanism, such that another experimenter replicating the experiment will be able to find similar clusters.\michel{connect to robustness postulate in causal inference} Such required property, although most commonly not explicitly stated, puts the clustering task in a causal inference perspective. Like for any causal inference problem, finding plausible causal generative models will require assumptions on the data generating mechanism. We can for example, try to exploit the ICM postulate to learn the structure of generative models from a causal perspective. As suggested in \citep{anticausal}, many real world datasets have an intuitive underlying causal structure that we may exploit to improve learning algorithms. For instance, in character recognition datasets such as MNIST, the character that a human intents to write is a cause for the observed hand-written character image.  

In this section, we assume the setting of \cref{fig:latent} in which observations are generated from latent variables trough a possibly (partially) unknown mechanism. We will postulate ICM holds between latent causes and the mechanism. We apply this strategy to specific unsupervised learning algorithms that have been used in a wide range of areas: Non-negative Matrix factorization (NMF) \citep{lawton1971self,paatero1994positive,lee1999learning} and the classical Gaussian mixture model for clustering. Finally, we draw a connection between our framework and GANs.  
\subsection{Non-negative Matrix factorization (NMF)}
Given a matrix of non-negative coefficients $\mathbf{X}$. The non-negative matrix factorization problem aims at finding two low rank matrices of non-negative coefficients $\mathbf{W}$ and $\mathbf{V}$, such that
$$
\mathbf{X}=\mathbf{W}\mathbf{V}^\top
$$
holds approximately. We will assume that $\mathbf{X}$ corresponds to observations that are generated by a generative model with latent variables representing ``sources" stored in $\mathbf{W}$, that are linearly mixed by an unknown linear process, whose coefficients are stored in $\mathbf{V}$. The interpretation of this model in terms of generative mechanism is application dependent but as been a motivation for using such method since early work \citep{anttila1995source}.
The NMF problem can be solved using several algorithms, but being NP-hard in general, it is difficult to get guaranties that the estimated factors are close to the true generative model.

\subsubsection{Permutation invariance hypothesis}
We assume that the matrices $\mathbf{W}$ and $\mathbf{V}$ are selected from permutation invariant distributions, in the sense that permuting the columns of $\mathbf{W}$ and the columns of $\mathbf{V}$ independently leads to a causal model which is as likely to occur in Nature. This implies that there is no particular relationship between a given column of $\mathbf{W}$ and its corresponding column in $\mathbf{V}$. 

\subsubsection{Contrast}
A way to assess whether the relationship between columns of $\mathbf{W}$ and of $\mathbf{V}$ is generic is to use the matrix squared $\ell_2$ norm of the observation matrix as a contrast:
\begin{equation}\label{eq:NMFtrace}
\tr \left[ \mathbf{X}\mathbf{X}^\top\right] = \tr \left[\mathbf{W}^\top \mathbf{W} \mathbf{V}^\top \mathbf{V}\right]
\end{equation}
The right hand side in the above equation indeed shows that the similarity matrices $\mathbf{W}^\top \mathbf{W}$ and $\mathbf{V}^\top \mathbf{V}$ are compared using the trace of their product, which is a scalar product for matrices. If the orderings of columns of $\mathbf{W}$ and $\mathbf{V}$ are unrelated (meaning other choices of orderings would be as likely to occur), one expects the eigenvectors of $\mathbf{W}^\top \mathbf{W}$ to have a generic orientation with respect to the eigenvectors of $\mathbf{V}^\top \mathbf{V}$, leading to "average" trace values. To make this more precise, we will introduce some notations.

Let $\mathbb{S}_n$ be the symmetric group that we will abusively identify to the set of $n\times n$ permutation matrices. We will thus abusively denote ${\mu}_{\mathbb{S}}$ the corresponding Haar measure on this group. As the group is finite and contains $n!$ elements, ${\mu}_{\mathbb{S}}$ assigns the probability $\frac{1}{n!}$ to each group element. 
In order to assess genericity of the trace value in \cref{eq:NMFtrace}, we evaluate the EGC:
\begin{equation}\label{eq:NMFtracegene}
\mathbb{E}_{P\sim {\mu}_{\mathbb{S}}} \tr\left[P \mathbf{W}^\top \mathbf{W} P^\top \mathbf{V}^\top \mathbf{V}\right].
\end{equation}
The result is provided in the following proposition\footnote{To shorten formulas, we use a modified version of the contrast so that it acts on centered matrices.}
\begin{prop}\label{prop:nmfcontr}
	Let $\tilde{\mathbf{M}}$ be the centered matrix obtained by subtracting the mean column $\bar{\mathbf{M}}$ from each column,
	\[
	\mathbb{E}_{P\sim {\mu}_{\mathbb{S}}}\!\tr [P \tilde{\mathbf{W}}^\top \tilde{\mathbf{W}} P^\top \tilde{\mathbf{V}}^\top \tilde{\mathbf{V}}]\!=\!\frac{\tr[\tilde{\mathbf{W}}^\top\tilde{\mathbf{W}}]\tr[\tilde{\mathbf{V}}^\top\tilde{\mathbf{V}}]}{n-1}.
	\]
\end{prop}
\subsubsection{Experiments}
\begin{figure*}
	\begin{subfigure}{.44\textwidth}
		\includegraphics[width=.85\textwidth]{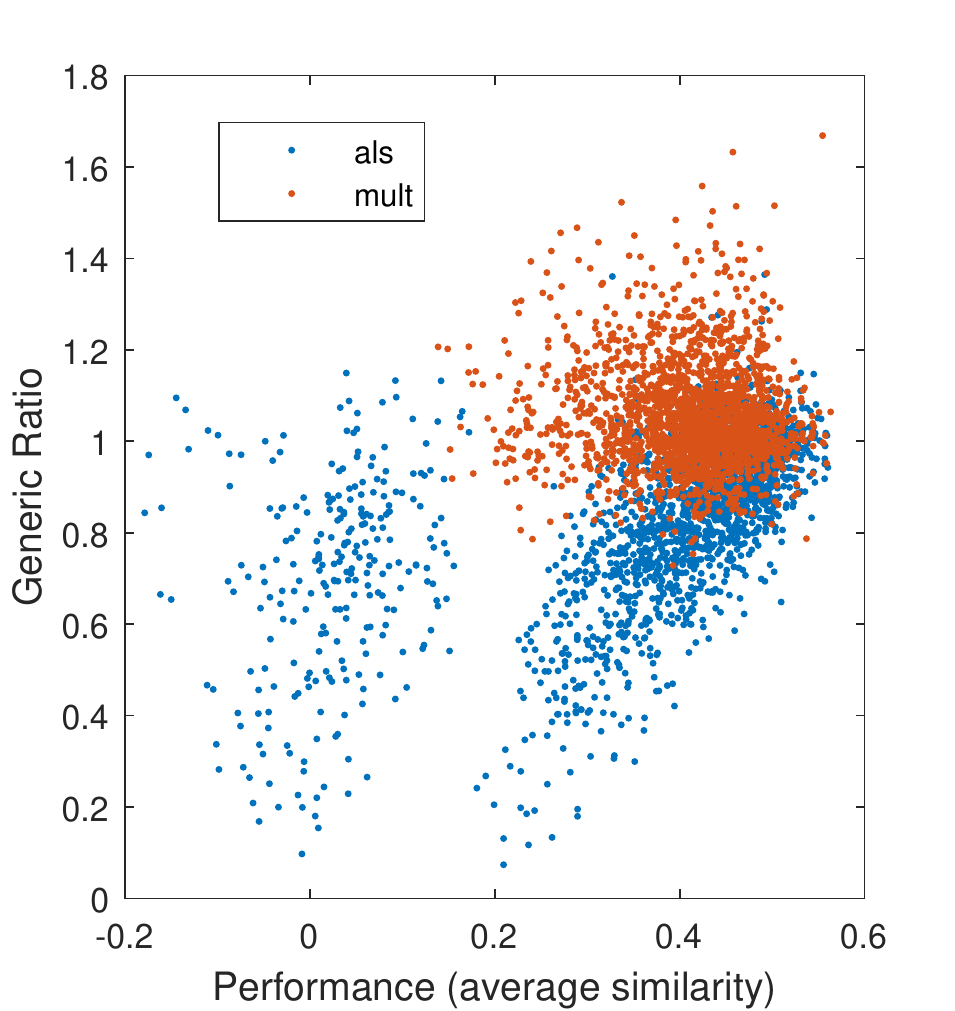}
		\subcaption{\label{fig:scatternmf}}
	\end{subfigure}
	\begin{subfigure}{.5\textwidth}
		\includegraphics[width=\textwidth]{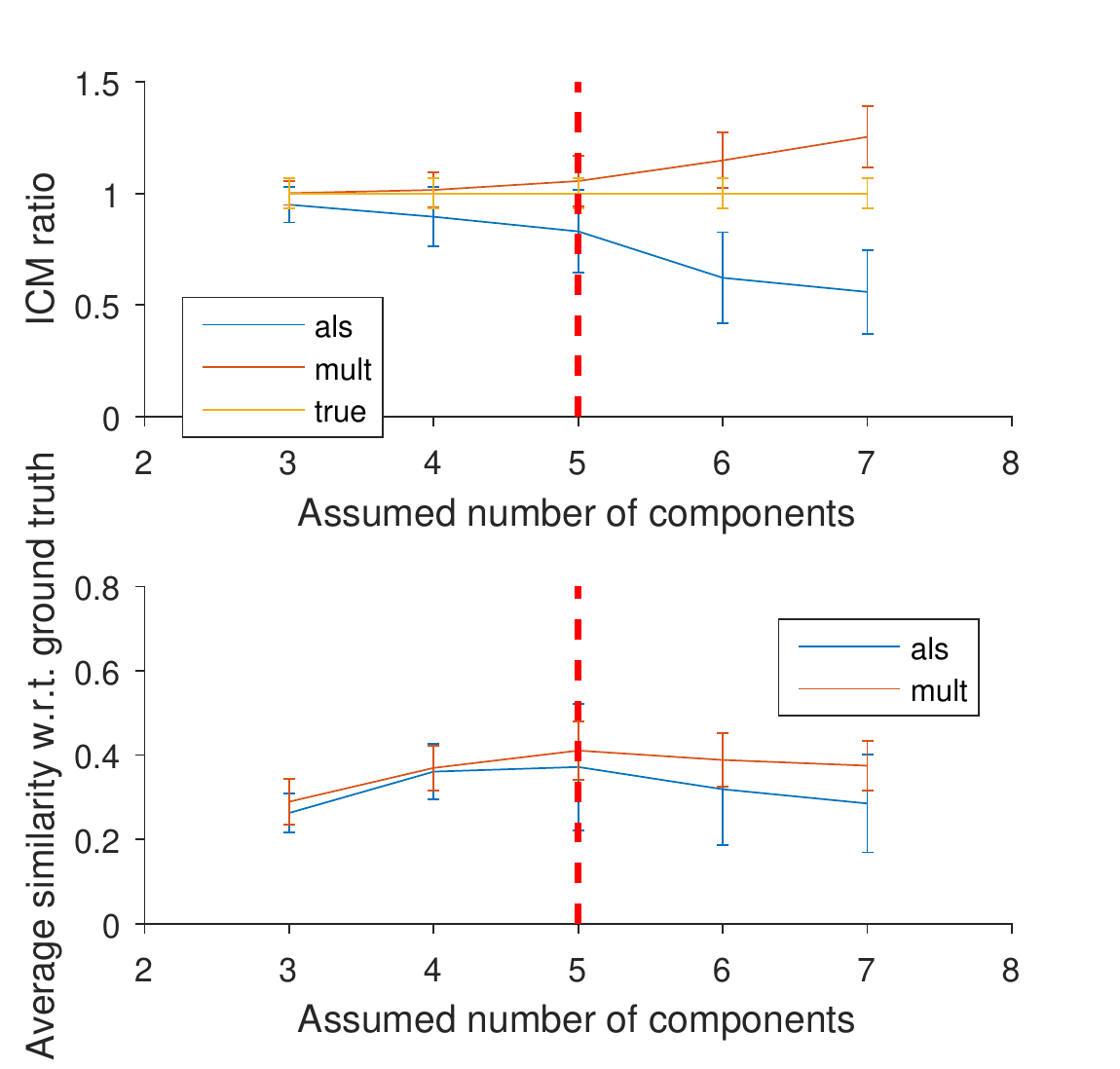}
		\subcaption{\label{fig:ncompnmf}}
	\end{subfigure}
	\vskip -.4cm
	\caption{(a) Scatter plot of performance versus generic ratio for both NMF algorithms. (b) Performance and generic ratio depending on the assumed number of components (dashed line indicates ground truth).\label{fig:nmfres}}
	\vskip -.2cm
\end{figure*}
In order to assess whether the group invariance framework is useful to better infer the NMF generative process, we simulated the model by generating $(20\time50)$ data matrices using 5 NMF components: $\mathbf{W}$ matrices were generated with i.i.d. coefficients uniformly distributed on the unit interval, and sparse matrices $\mathbf{V}$ by selecting non-zero coefficients by sampling from i.i.d. Bernoulli variables with probability $.1$ (sparsity of one factor matrix tends to increase the identifiability of NMF models \citep{donoho2003does}), the value of the selected non-vanishing coefficients were drawn from i.i.d. uniform distributions on the unit interval. A small i.i.d. uniform additive noise was added to the data matrix. First, we assumed the number of components was known and we quantified the performance of the algorithm by computing the average cosine similarity between the columns of ground truth $\mathbf{W}$ and the corresponding best matching columns estimated by the NMF algorithms. Simultaneously, we estimated the generic ratio of the ground truth model and the estimated models. Results for two different algorithms (alternating least squares, 'als', and multiplicative updates, 'mult') for 2000 simulations are provided in \cref{fig:scatternmf}. They show that both algorithms tend to introduce dependencies between the estimated latent variables and the mechanism, as the generic ratio tends to be larger than one for 'mult', while the ratio is smaller than one for 'als'. In addition, while the 'mult' performance tends to be less variable across trials, and leads to generic ratios close to one, the 'als' algorithm fails frequently  and leads to particularly low values of the generic ratio. We next assumed that the number of components is unknown, then \cref{fig:ncompnmf} show the evolution of performance and generic ratio depending on the number of estimated components. Overestimation of the number of components leads to a generic ratio that progressively departs from one, suggesting that it can be used to indicate a misspecification of the model. Overall these results suggest that the generic ratio can be used to detect failures of NMF algorithms, assuming that the generic causal model respects the ICM postulate. In addition, the behavior of the generic ratio is algorithm specific, introducing  perturbations in the estimated parameters that can be quantified and exploited to improve the algorithm.
 
\subsection{Clustering}
Consider the following classical Gaussian Mixture Model of the observed multivariate data $\boldsymbol{X}$ using latent variable $Z$.
\begin{eqnarray}
Z & \sim & \mbox{Mult}(\pi_1,\pi_2,\cdots,\pi_K)\,,\\
\boldsymbol{X}|\{z=k\} & \sim &\mathcal{N}(\boldsymbol{\mu}_k,\boldsymbol{\Sigma}_k)\,,
\end{eqnarray}
where $z$ indicates the cluster membership of one observation, and $\boldsymbol{\mu}_k$, $\boldsymbol{\Sigma}_k$ are means and covariances of the $p$-dimensional Gaussian distribution of each cluster.

\subsubsection{Invariance hypothesis}
To get an insight of what form of genericity is relevant for such generative model, imagine the collected data reflects the phenotype of different subspecies of plants (similarly to the popular Iris dataset). Each cluster mean $\boldsymbol{\mu}_k$ reflects the average characteristics of the species, while the covariance matrices $\boldsymbol{\Sigma}_k$ express the variations of these characteristics across the subpopulation. If we assume that each subspecies has emerged independently (say on different continents) and that they never interacted with each other (no competition for resources), we suggest that the variability within each subspecies should be unrelated to the variations across species. As a consequence, we could imagine that randomizing the properties of $\boldsymbol{\mu}_k$'s while keeping $\boldsymbol{\Sigma}_k$'s constant would lead to a model as likely to have been generated by Nature as the observed dataset.

This can be made quantitative by representing the mixture using the following generative model
\[
\boldsymbol{X}=\boldsymbol{V}+\boldsymbol{M}\,,\,\boldsymbol{V}|\{z=k\}\!\!\sim\!\! \mathcal{N}(0,\boldsymbol{\Sigma}_k)\,,\,\boldsymbol{M}|\{z=k\}\!\!=\!\!\boldsymbol{\mu}_k.
\]
The model is thus decomposed into the sum of a (intra-cluster) variability component $\boldsymbol{V}$ and a cluster mean component $\boldsymbol{M}$. 

\begin{figure*}
	\begin{subfigure}{.33\textwidth}
		\includegraphics[width=1.15\textwidth]{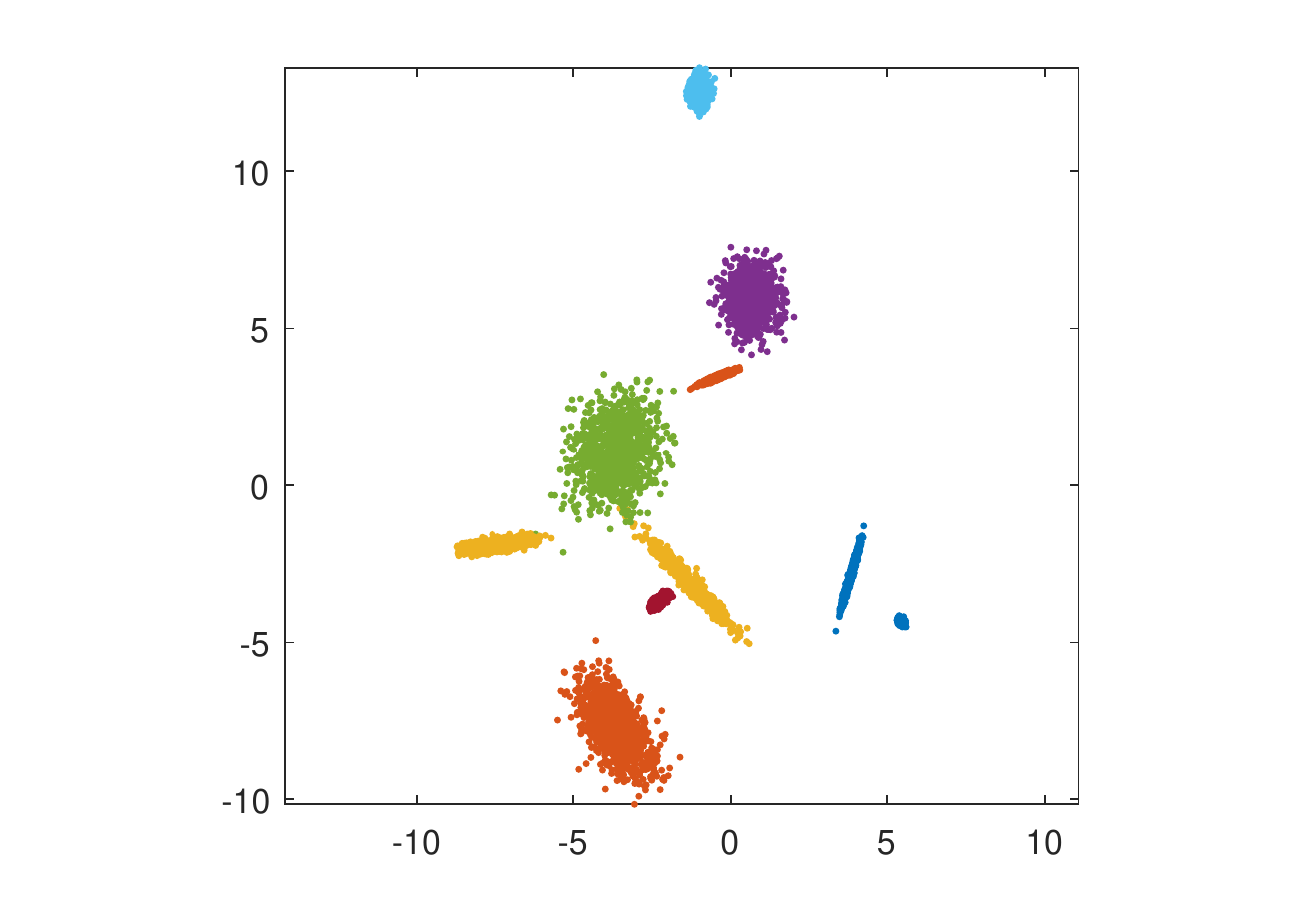}
		\subcaption{\label{fig:clustscatt1}}
	\end{subfigure}
	\begin{subfigure}{.33\textwidth}
		\includegraphics[width=.95\textwidth]{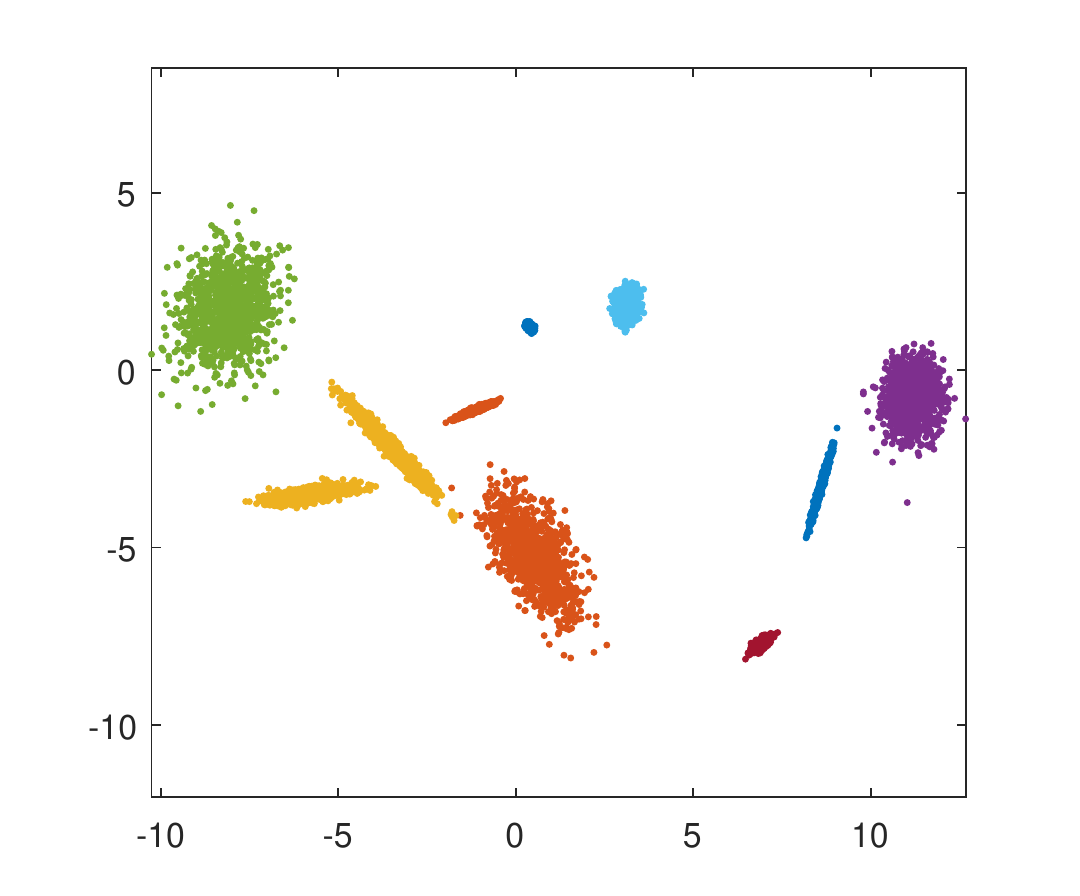}
		\subcaption{\label{fig:clustscatt2}}
	\end{subfigure}
\hspace*{-.3cm}
	\begin{subfigure}{.3\textwidth}
		\includegraphics[width=1.05\textwidth,height=.85\textwidth]{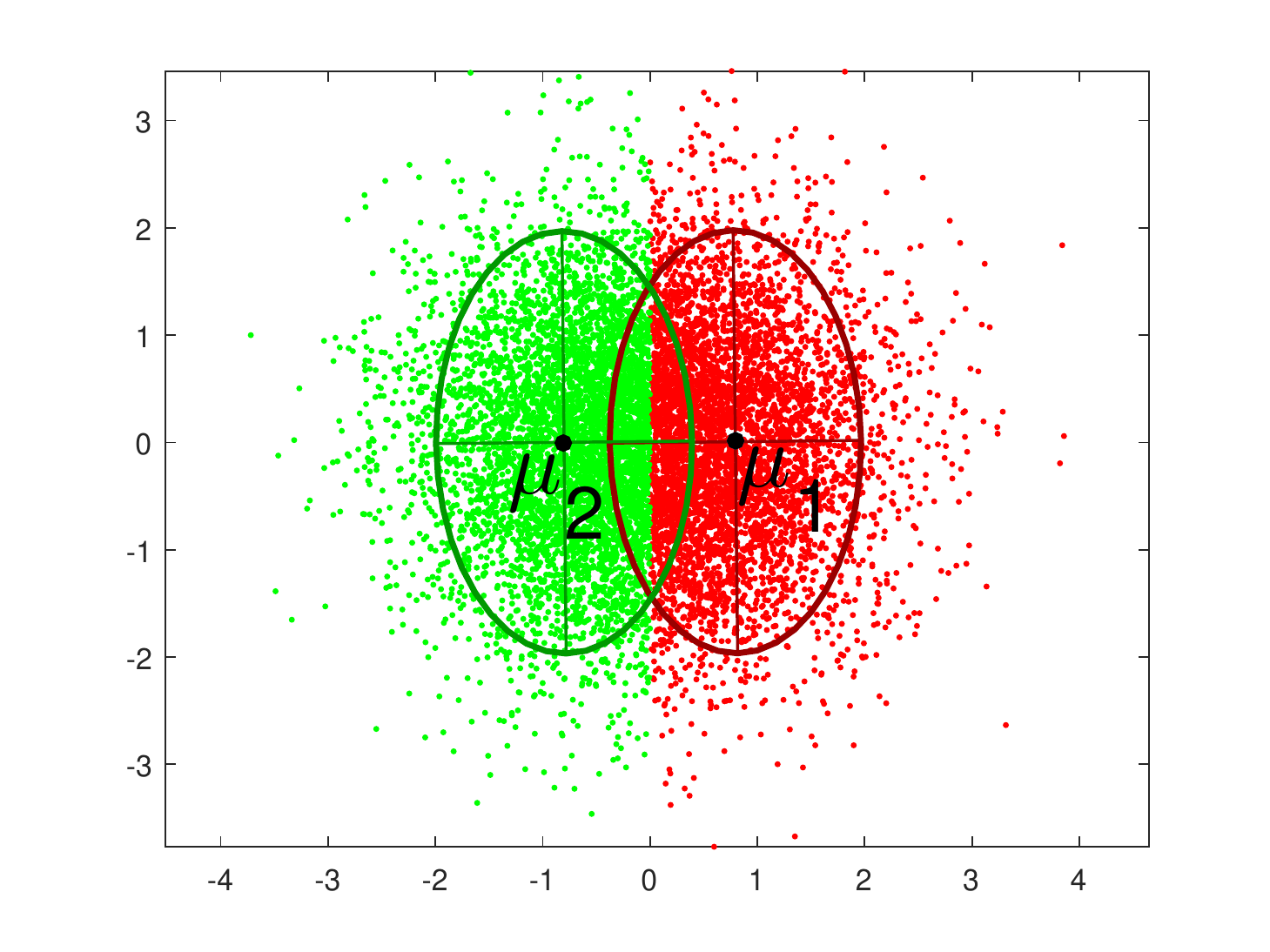}
		\subcaption{ \label{fig:clustsuspic}}
	\end{subfigure}
	\caption{(a) Cluster data generated with random parameters (projected on 2 components) (b) Data in (a) after a generic transformation (c) Suspicious dependency between cluster means and covariances in the case of a mispecified number of clusters}
\end{figure*}

 We then choose $O(p)$ as generic group: and p-dimensional orthogonal matrices from this group act on the mean vectors by left multiplication to the variable $\boldsymbol{M}$, before $\boldsymbol{V}$ is added. Application of one generic transformation results in clusters with the same intra-cluster variability as the original data, but whose locations in the feature space have been randomized, as illustrated on \cref{fig:clustscatt1,fig:clustscatt2} with an 5-dimensional feature space and 10 clusters. In this illustration, the structure of the observations does not seem to be affected by the transformation, suggesting the original data is "typical" in some sense. This makes sense as mean and covariance parameters have been drawn independently at random. However, there are simple pathological examples where a clustering algorithm can fail to capture the underlying structure of the data and generate an atypical dependency between means and covariances. Assume for example that, focusing on one single Gaussian cluster, a clustering algorithm fails to identify a single cluster and instead cuts it in two clusters. This situation illustrated on \cref{fig:clustsuspic} shows an interesting dependency between the centroids of the two clusters and their within cluster empirical covariance matrices: the difference between centroids is oriented in the direction (eigenspace) of smallest variance. We postulate that such suspicious dependencies may appear when the clustering algorithm fails to capture the causal structure of the data. 
 
 \subsubsection{Contrast}
 To detect such suspicious dependencies in the inferred generative model using the group theoretic approach, we propose the following 4th order tensor contrast
\[
C(\boldsymbol{X}) = \mathbb{E}\tr \left[\boldsymbol{X}\boldsymbol{X}^\top \boldsymbol{X}\boldsymbol{X}^\top\right]\,.
\]
One justification for using this contrast is the following.
\begin{prop}\label{prop:clustgene}
	Let $\boldsymbol{X}$ be a centered multivariate Gaussian mixture random variable, and $\left\langle C\right\rangle_{\boldsymbol{\mu},\boldsymbol{\Sigma}}$ the generic contrast obtained by random orthogonal transformation applied to cluster means, then 
\[
C(\boldsymbol{X})-\left\langle C\right\rangle_{\boldsymbol{\mu},\boldsymbol{\Sigma}}\!=\!4 \sum_k \!\pi_k\!\! \left(\boldsymbol{\mu}_k^\top \boldsymbol{\Sigma}_k \boldsymbol{\mu}_k\!-\! \|\boldsymbol{\mu}_k\|^2 \frac{\tr\left[\boldsymbol{\Sigma}_k\right]}{p}\right).
\]
\end{prop}
Indeed, this result shows that differences between the contrast of the data and the EGC quantify the alignment between the cluster means $\boldsymbol{\mu}_k$ and the principal axes of the covariance matrices $\boldsymbol{\Sigma}_k$. We then test empirically whether the resulting generic ratio can indicate suspicious dependencies in the solution of clustering algorithms.

\begin{figure*}
	\centering
	\begin{subfigure}{.34\textwidth}
			\includegraphics[width=\textwidth]{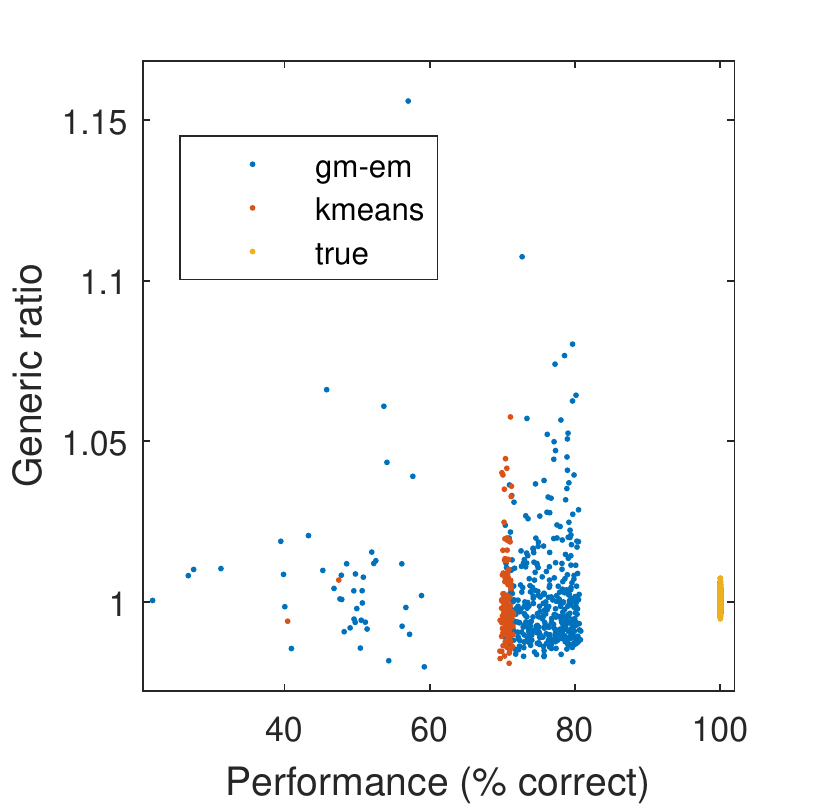}
			\subcaption{\label{fig:clustresscat}}
	\end{subfigure}
	\begin{subfigure}{.6\textwidth}
		\includegraphics[width=\textwidth]{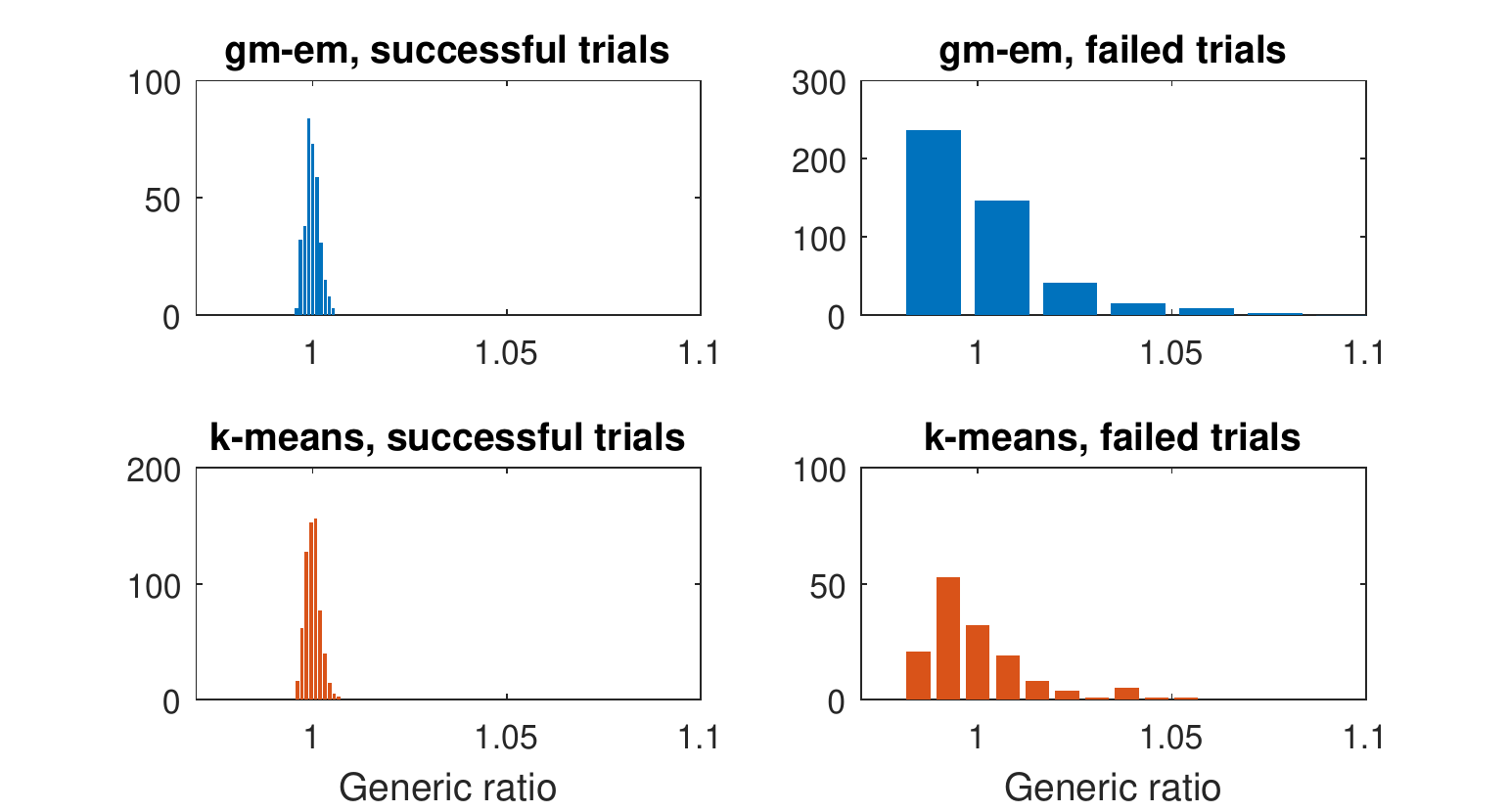}
		\subcaption{\label{fig:clustreshist}}
	\end{subfigure}
	\caption{(a) Clustering performance of 'gm-em' and 'kmeans' algorithms against generic ratio for 800 simulations; the yellow ground truth points on the very right hand-side indicate how the ground truth generic ratio concentrates (note they hide points of successful trials for both algorithms). (b) Distribution of the generic ratios for both algorithms in case of successful trials (performance exceeds 99\%) and failed trials (performance below 99\%). \label{fig:clustres}}
\end{figure*}

\subsubsection{Experiments}
We test this approach to detect bad clustering of a simulated dataset. We generate 5 random clusters in a 20 dimensional space. The cluster means are drawn at random from an isotropic Gaussian distribution with standard deviation 2. Cluster covariances are generated with random axes (with isotropic distribution) and eigenvalues. We test the performance of two clustering algorithm: K-means ('kmeans') and the Expectation Minimization algorithm based on the simulated Gaussian Mixture model ('gm-em'). The scatter plot shown on \cref{fig:clustresscat} suggests that the generic ratios are broadly distributed on the interval $[0.98,1.1]$ when the algorithms do not reach a good estimation of the original clusters. Comparison of the distributions of the generic ratio in case of success and failure of the clustering shown on \cref{fig:clustreshist} shows a much more concentrated distribution when the clusters are correctly retrieved. This suggests that a generic ratio far from one indeed witnesses the failure of the algorithm to cluster the data properly and could  be exploited to improve the performance of clustering algorithms.

\subsection{Learning causal generative models}
There as been a recent interest in learning complex generative models from data using deep neural networks. This has led to the design of Generative Adversarial Networks (GANs) \citep{goodfellow2014generative} which were able to produce impressively realistic synthetic images using several image datasets. The principle of GANs is to oppose the generative model $G$ to an adversarial discriminative model $D$: while the goal of $D$ is to differentiate real data from data generated by $G$, the goal of $G$ is to mislead $D$ such that it mistakenly considered generated images as real. We will argue that there is an intriguing connection between GANs and our group theoretic framework. First let us introduce an example GAN more formally. The generator's output random variable $\boldsymbol{X}\sim p_g$ is generated by applying a parametric mapping to a multivariate input noise variable $\boldsymbol{Z}$ using the function $G(\boldsymbol{Z},\boldsymbol{\theta}_g)$, typically implemented by a deep neural network. The distribution of $\boldsymbol{Z}$ is fixed and consists in $K$ i.i.d. uniformly distributed variables (say on the unit interval), and the parameters of the mapping (usually consisting in the synaptic weights of a deep neural network), are optimized. The optimization is done simultaneously with the parameters of a discriminator $D(\boldsymbol{X},\boldsymbol{\theta}_d)$ which outputs the probability that a sample $\boldsymbol{x}$ came form the data rather that $p_g$.

In principle, the optimization aims at solving the two-player minimax game $\min_{\boldsymbol{\theta}_g} \max_{\boldsymbol{\theta}_d} v(G,D)$ where
\begin{multline}
	v(G,D)=\mathbb{E}_{X\sim p_{data}}\!\left[\log D(\boldsymbol{X})\right]\\+\mathbb{E}_{\boldsymbol{X}\sim p_g}\left[\log (1-D(\boldsymbol{X}))\right]\,.
\end{multline}
If we focus on the optimization of the generative model, for a fixed discriminator it should minimize
\[
\mathbb{E}_{X\sim p_g}\left[\log (1-D(\boldsymbol{X}))\right]\,.
\]
However, as discussed in \citep{goodfellow2014generative,uehara2016generative}, it is heuristically preferable to maximize instead
\[
\mathbb{E}_{\boldsymbol{X}\sim p_g}\left[\log (D(\boldsymbol{X}))\right]\,.
\]
In the ideal case (yet in principle not reached in practice), the generator would be indistinguishable from the data by $D$ such that.
\begin{equation}\label{eq:GANequal}
\mathbb{E}_{\boldsymbol{X}\sim p_g}\left[\log (D(\boldsymbol{X}))\right]=
\mathbb{E}_{\boldsymbol{X}\sim p_{data}}\left[\log (D(\boldsymbol{X}))\right]
\,.
\end{equation}
\michel{either justify why now we use an expectation on both sides or find an expression for a single sample looking at the optimization procedure in GAN papers (training with one sample)}
Interestingly, this last equation can be interpreted as an averaged genericity equation. Indeed, as $\boldsymbol{X}=G(\boldsymbol{Z})$ with the coefficients of $\boldsymbol{Z}$ being uniform i.i.d., then it can be considered as a generative model whose cause $\boldsymbol{Z}$ has an invariant distribution on the unit hypercube $[0,\,1]^{K}$, the generic group being the group combining translations modulo 1 for each component \footnote{this is the direct product of the groups associated to each components}. Taking $D$ as a contrast, we can fix an input data point $z_0$ and rewrite the left-hand side \cref{eq:GANequal} as an EGC
\begin{multline}
\mathbb{E}_{\boldsymbol{X}\sim p_g}\left[\log (D(\boldsymbol{X}))\right]=\mathbb{E}_{g\sim \mu_{\G}}\left[\log (D(G(gz_0)))\right]\\=\left\langle D \right\rangle_{G,z_0}\,.
\end{multline}
As a consequence, the goal of the optimization of $G$ can be interpreted as building a generative model that satisfies a genericity equation based on the contrast defined by the discriminator $D$. Overall, training a GAN can be considered as applying the group theoretic framework to find a robust generative model, such that applying a generic transformation to the cause will lead to a corresponding output that will still look typical. In the case of a face generator for example, we want that when applying a generic transformation to an input $z_0$ that generates a true face image, the output still looks like a face afterwards. The characteristic ``looks like a face" is quantified by the contrast implemented by the discriminator $D$. As a consequence, a GAN can be seen in a group theoretic framework as a causal generative model whose structure and corresponding contrast are tuned based on a dataset to match the ICM postulate.

This interpretation, together with the empirical success of GANs, suggests that the principle of genericity can be used to learn robust generative models. The robustness may be further enforced by applying this principle not only at the input of the deep neural network implementing the generative model, but also at the level of intermediate representations of this network. This might be a way to train better GANs who capture reliably the hierarchical, and potentially causal, structure of the data.

%% file: SI_ICML17bis.tex
\newpage
\onecolumn
{\Huge Supplementary Information}
\section*{Appendix A: Elements of group theory}
Sets equipped with a group structure have attracted interest from the machine learning community because they can model the data structure of complex domains \citep{fukumizu2009characteristic}. We will introduce concisely the concepts and results of group theory necessary to this paper. The authors can refer for example to \citep{tung1985group,wijsman1990invariant,Eaton1989} for more details.
\begin{defn}[Group]
A set $\mathcal{G}$ is said to form a group if there is an operation `*', called group multiplication, such that:
\begin{enumerate}
\item For any $a,b\in \G$, $a*b\in \G$.
\item The operation is \textit{associative}: $a*(b*c)=(a*b)*c$, for all $a,b,c\in\mathcal{G}$,
\item  There is one identity element $e\in\mathcal{G}$ such that, $g*e=e$ for all $g\in\mathcal{G}$,
\item  Each $g\in\mathcal{G}$ has an inverse $g^{-1}\in\mathcal{G}$ such that, $g*g^{-1}=e$.
\end{enumerate}
A subset of $\mathcal{G}$ is called a subgroup if it is a group under the same multiplication operation.
\end{defn}
The following elementary properties are a direct consequence of the above definition:
$e^{-1}=e$,
$g^{-1}*g=e$,
$e*g=g$, for all $g\in\mathcal{G}$.

Among others, classical groups of interest in this paper are the permutations group $\mathbb{S}(n)$ and the \textit{general linear group} $GL(n)$ of all real nonsingular $n\times n$ matrices equipped with matrix multiplication. The matrix representations of the real orthogonal group $O(n)$ of isometries and of the real special orthogonal group $SO(n)$ of rotations are subgroups of $GL(n)$. As in these two examples, many groups can be considered as functions \textit{acting} on an input space:
\begin{defn}[Action]
Let $\G$ be a group and $\X$ a space. An action of $\G$ on $\X$ to the left is a function $a:\G\times\X\rightarrow\X,(g,x)\mapsto g.x$ such that:
\begin{enumerate}
\item $e.x=x$, for all $x\in\X$
\item $g_2.(g_1.x)=(g_2*g_1).x$, for all $g_1,g_2\in\G,\,x\in\X$
\end{enumerate}
If $g*x=x$, $x$ is called a fixed point of $g$. We will call the subgroup of elements fixing $x$, $\G_{x}=\{g\in\G,g*x=x\}$, the isotropy subgroup or stabilizer of $x$ in $\G$.
$\G$ is said to act freely if $gx\neq x$, for all $g\in\G\setminus\{e\}$ and $x\in\X$.
\end{defn}
Due to the properties of group actions, associativity rules can be applied to all group actions and group multiplications of a given expression, such that we can do not need to put any symbol for binary operations between group/space elements. For example, we will thus simply denote $g_1.((g_2*g_3).x)$ by $g_1g_2g_3x$.

We will always consider group actions to the left in this paper, such that we will simply call them group action. It is easy to show that $\mathbb{S}(n)$ and its subgroups act on the set $\{1,..,n\}$ by permuting its elements, as well as on $n$-tuples from arbitrary sets. $GL(n)$ and its subgroups act as linear functions on the vector space $\mathbb{R}^n$.

\begin{defn}[Topological group]
\label{def:topogroup}
A locally compact Hausdorff topological group is a group equipped with a locally compact Hausdorff topology such that:
\begin{itemize}
\item $\G\rightarrow \G: x\mapsto \inv{x}$ is continuous,
\item $\G\times\G\rightarrow \G: (x,y)\mapsto x.y$  is continuous (using the product topology).
\end{itemize}
The $\sigma$-algebra generated by all open sets of G is called the Borel algebra of $\G$. 
\end{defn}

\begin{defn}[Invariant measure]
Let $\G$ be a topological group according to definition~\ref{def:topogroup}. Let $K(\G)$ be the set of continuous real valued functions with compact support on $\G$. A radon measure $\mu$ defined on Borel subsets is left invariant if for all $f\in K(\G)$ and $g\in\G$
$$
\int_G f(\inv{g} x)d\mu(x)=\int_G f(x)d\mu(x)
$$
Such a measure is called a \textit{Haar measure}.
\end{defn}
A key result regarding topological groups is the existence and uniqueness up to a positive constant of the Haar measure \citep{Eaton1989}. Whenever $\G$ is compact, the Haar measures are finite and we will denote $\mu_\G$ the unique Haar measure such that $\mu_\G(\G)=1$, defining an invariant probability measure on the group.
\begin{defn}[Proper Mapping (\citep{wijsman1990invariant},theorem 2.2.3)]
Let $X$ and $Y$ be locally compact spaces and $f:X\rightarrow Y$ continuous, $f$ is proper if $f^{-1}[K]$ is compact for every compact $K\subset Y$
\end{defn}

\section*{Appendix B: ICM for linear dynamical systems}
Assume now that our cause effect pairs $\textbf{X}$ and $\textbf{Y}$ are weakly stationary time series where $\textbf{Y}$ is the output of a Linear Time Invariant Filter with input $\textbf{X}$:
\begin{equation}\label{eq:conv}
\textstyle{\textbf{Y}=\{\sum_{\tau\in \mathbb{Z}} h_\tau X_{t-\tau} \,\}=\textbf{h}\star\textbf{X}, }
\end{equation}
where $\textbf{h}$ denotes the {\it impulse response} of the filter being convolved to the input to provide the output. Then \citep{shajarisale2015} postulate the following:
\begin{post}[Spectral Independence Criterion (SIC)]
	Let $S_{xx}$ be the Power Spectral Density (PSD) of a cause $\textbf{X}$ and ${\bf h}$ the system impulse response of the causal system of (\ref{eq:conv}), then
	\begin{gather}\label{eq:sic}
	\textstyle{\int_{-1/2}^{1/2} S_{xx}(\nu) |\widehat{{\rm h}}(\nu)|^2 d\nu  = \int_{-1/2}^{1/2} S_{xx}(\nu) d\nu \cdot \int |\widehat{\rm h} (\nu)|^2 d\nu \,,}
	\end{gather}
	holds approximately, where $\widehat{\rm h} $ denotes the Fourier transform of ${\bf h} $.
\end{post}
It can be shown that (\ref{eq:sic}) is violated in the backward direction under mild assumptions. By using the power of the time series ${\bf Y}$  in (\ref{eq:conv}) as a contrast, we can retrieve the expression of the Spectral Independence for a proper CMU (the proof is in Appendix B). Since PSD's are even functions for real valued signals, the generic transformations are defined by their action on positive frequencies and the negative frequencies are built by symmetry.
\begin{prop}
	\label{thm:CMU-SIC}
	Let $\G$ be the group of modulo 1/2 translations that acts on the PSD by shifting its graph for positive frequencies ($\nu \in [0,\,1/2]$) while the graph for negative frequencies is defined so that the transformed PSD is even. Using the total power as a contrast, $\G-$genericity is equivalent to SIC. 
\end{prop}

\subsection*{Proof of proposition~\ref{thm:CMU-SIC}}
Suppose that for a given mechanism $m$ and given input $S_{xx}$ the $\G-$genericity assumption is satisfied. Noticing that $\mu_{\G}$ is the uniform probability measure over $[0,\,1/2]$. This amounts to
\begin{gather*}
\int_{-1/2}^{1/2} S_{xx}(\nu)|\widehat{{\rm h}}(\nu)|^2 d\nu=\int_{0}^{1/2} \left(2\int_{0}^{1/2} |\widehat{\rm h}(\nu)|^2 S_{xx}(\nu-g)\mu_G(g)d\nu \right) dg\\
= 4\int_{0}^{1/2}\int_{0}^{1/2} |\widehat{\rm h}(\nu)|^2 S_{xx}(\nu-g)d\nu dg\\
= 4\int_{0}^{1/2} |\widehat{\rm h}(\nu)|^2 \left(\int_{0}^{1/2}S_{xx}(\nu-g)dg \right) d\nu\\
=\int_{-1/2}^{1/2} S_{xx}(\nu) d\nu \cdot  \int_{-1/2}^{1/2} |\widehat{\rm h} (\nu)|^2 d\nu
\end{gather*}
This corresponds to the formula of the SIC postulate.\ $\square$
\section*{Appendix C: proofs of main text results}

\subsection*{Equation \ref{eq:expectU}}
We will use $\mathbb{E}_U$ to denote the expectation when $U$ is drawn from the distribution on $SO(n)$ (or sometimes just $\mathbb{E}$ when it does not lead to confusion).
Since $A$ is symmetric then we can decompose it as $A = V^T D V$, with $D$ diagonal and $V\in SO(n)$. Then 
\[\mathbb{E}_U\tr \left(U^T A U B\right) = \mathbb{E}_{VU}\tr \left((VU)^T D VU B\right) = \mathbb{E}_U\tr \left(U^T D UB\right),\]
where we substituted $VU$ for $U$ in the expectation and used the translation invariance of the Haar measure.
As a consequence, 
\[\mathbb{E}\tr \left(U^T A U B\right)=\mathbb{E}\tr \left(D UBU^T\right)=\sum d_{kk} \mathbb{E} \left(U B U^T\right)_{kk}\]
by cyclic invariance and linearity of the trace. We claim that the values of the diagonal elements of $\mathbb{E}\left(U B U^T\right)$ are all equal. Indeed, let $P_{i,k}\in SO(n)$ be the matrix permuting coordinates $i$ and $k$. Then
\[\mathbb{E} \left(U B U^T \right)_{kk}=
\left(P_{i,k}\left(\mathbb{E}\left(U B U^T\right)\right)P_{i,k}^T\right)_{ii}=
\mathbb{E}\left(P_{i,k}U B\left(P_{i,k}U\right)^T\right)_{ii}=
\mathbb{E}\left(UBU^T\right)_{ii}\]
again by substitution the translation invariance of the Haar measure.
Therefore, for all $k$, 
\[\mathbb{E}\left(U B U^T\right)_{kk}=
\textstyle\frac{1}{n} \mathbb{E}\tr \left(U B U^T \right)=
\frac{1}{n}\mathbb{E}\tr \left(U^T U B\right)=
\frac{1}{n} \tr (B)\]
since $U$ is orthogonal. Finally we get 
\[\mathbb{E}_U\tr \left(U^T A U B\right)=\textstyle\frac{1}{n}\tr (B)\sum d_{kk}=\textstyle\frac{1}{n}\tr (B)\tr (A).\]

\subsection*{Equation~\ref{eq:ratio}}
This is because for every fixed $\tilde{\alpha}$, the denominator of \eqref{eq:ratio} is a constant
and thus the conditional expectation of \eqref{eq:ratio}, given $\tilde{X}$,
is $1$. Hence, the overall expectation is $1$.

\subsection*{Proof of proposition~\ref{prop:nmfcontr}}
We will denote $\Ones$ the  $n\times n$ all-ones matrix (whose elements are all equal to one) and $\ones$ the $n$-dimensional all-ones column vector, such that 
\[
\Ones=\ones \ones^\top\,.
\]
We first prove the following.
\begin{lem}\label{lem:schurperm}
	Let $A$ be a $n\times n$ real matrix, $P$ a random permutation matrix Haar distributed. Then
	\[
	\mathbb{E}_{P\sim\mu_{\mathbb{S}}} P A P^\top=B 
	\begin{bmatrix}
	\alpha {\bf I}_{n-1} & {\bf 0}\\
	{\bf 0} &\lambda
	\end{bmatrix}
	B^\top
	\]
	where $\alpha=\frac{1}{n-1}\left(\tr(A) -\frac{1}{n} \tr (\Ones A) \right) $, $\lambda=\frac{1}{n} \tr (\Ones A)$ and $B$ is an arbitrary orthogonal matrix whose last column is $\frac{\ones}{\sqrt{n}}$.
\end{lem} 
\begin{proof}[Sketch of the proof]
	We first notice that $\mathbb{E}_{P\sim\mu_{\mathbb{S}}} P A P^\top$ commutes with all permutation matrices. 
	Since the group of permutation matrices has two invariant subspaces: the linear span of $\ones$ and its orthogonal complement, we can use Schur's lemma implying that $\mathbb{E}_{P\sim\mu_{\mathbb{S}}} P A P^\top$ is a multiple of the identity on each subspace. This implies that in the basis of the columns of $B$, $\mathbb{E}_{P\sim\mu_{\mathbb{S}}} P A P^\top$ is diagonal. The values of $\alpha$ and $\lambda$ are deduced from computing the traces of $\mathbb{E}_{P\sim\mu_{\mathbb{S}}} P A P^\top$ and $\Ones\cdot\mathbb{E}_{P\sim\mu_{\mathbb{S}}} P A P^\top$
\end{proof}

\begin{proof}[Proof of proposition~\ref{prop:nmfcontr}]
	Because of the invariance of the Haar measure, any permutation matrix $Q$ we have
	\[
	\mathbb{E}_{P\sim \mu_{\mathbb{S}}} \left[\tr\left(P \mathbf{W}^\top \mathbf{W} P^\top \mathbf{V}^\top \mathbf{V}\right)\right]=	\mathbb{E}_{P\sim \mu_{\mathbb{S}}} \left[\tr\left(QP \mathbf{W}^\top \mathbf{W} P^\top Q^\top \mathbf{V}^\top \mathbf{V}\right)\right].
	\]
	Thus
	\[
	\mathbb{E}_{P\sim \mu_{\mathbb{S}}} \left[\tr\left(P \mathbf{W}^\top \mathbf{W} P^\top \mathbf{V}^\top \mathbf{V}\right)\right]=	\mathbb{E}_{P\sim \mu_{\mathbb{S}},Q\sim \mu_{\mathbb{S}}} \left[\tr\left(P \mathbf{W}^\top \mathbf{W} P^\top Q^\top \mathbf{V}^\top \mathbf{V} Q\right)\right].
	\]
	Using lemma~\ref{lem:schurperm}, the expression inside the trace becomes a product of two diagonal matrices, leading to
	\[
	\mathbb{E}_{P\sim \mu_{\mathbb{S}}} \left[\tr\left(P \mathbf{W}^\top \mathbf{W} P^\top \mathbf{V}^\top \mathbf{V}\right)\right]=	(n-1)\alpha\beta+\lambda\gamma\,.
	\]
	with $\alpha=\frac{1}{n-1}\left(\tr(\mathbf{W}^\top \mathbf{W}) -\frac{1}{n} \tr (\Ones \mathbf{W}^\top \mathbf{W}) \right) $, $\beta=\frac{1}{n-1}\left(\tr(\mathbf{V}^\top \mathbf{V}) -\frac{1}{n} \tr (\Ones \mathbf{V}^\top \mathbf{V}) \right) $, $\lambda=\frac{1}{n} \tr (\Ones \mathbf{W}^\top \mathbf{W})$ and $\gamma=\frac{1}{n} \tr (\Ones \mathbf{V}^\top \mathbf{V})$.
	We can then easily check that the two terms in $\alpha$ and $\beta$ can be factored to lead to the expected expression.
\end{proof}

\subsection{Proof of proposition~\ref{prop:clustgene}}

\begin{proof}
	Decomposing $ \boldsymbol{X}  \boldsymbol{X} ^\top  \boldsymbol{X}  \boldsymbol{X} ^\top$ using $ \boldsymbol{X} =\boldsymbol{M}+\boldsymbol{V}$ 
	\begin{multline*}
		 \boldsymbol{X}   \boldsymbol{X} ^\top  \boldsymbol{X}   \boldsymbol{X} ^\top =		\left( \boldsymbol{V}  \boldsymbol{V} ^\top +    \boldsymbol{V}  \boldsymbol{M} ^\top+ \boldsymbol{M}  \boldsymbol{V} ^\top+ \boldsymbol{M}  \boldsymbol{M} ^\top \right)^2\\
		 \boldsymbol{V}  \boldsymbol{V} ^\top  \boldsymbol{V}  \boldsymbol{V} ^\top+ \boldsymbol{V}  \boldsymbol{M} ^\top  \boldsymbol{V}  \boldsymbol{M} ^\top+ \boldsymbol{M}  \boldsymbol{V} ^\top  \boldsymbol{M}  \boldsymbol{V} ^\top+ \boldsymbol{M}  \boldsymbol{M} ^\top  \boldsymbol{M}  \boldsymbol{M} ^\top\\
		+ \boldsymbol{V}  \boldsymbol{V} ^\top  \boldsymbol{V}  \boldsymbol{M} ^\top+  \boldsymbol{V}  \boldsymbol{V} ^\top  \boldsymbol{M}  \boldsymbol{V} ^\top +  \boldsymbol{V}  \boldsymbol{V} ^\top  \boldsymbol{M}  \boldsymbol{M} ^\top\\
		+ \boldsymbol{V}  \boldsymbol{M} ^\top  \boldsymbol{V}  \boldsymbol{V} ^\top+ \boldsymbol{V}  \boldsymbol{M} ^\top  \boldsymbol{M}  \boldsymbol{V} ^\top + \boldsymbol{V}  \boldsymbol{M} ^\top  \boldsymbol{M}  \boldsymbol{M} ^\top\\
		+ \boldsymbol{M}  \boldsymbol{V} ^\top  \boldsymbol{V}  \boldsymbol{V} ^\top + \boldsymbol{M}  \boldsymbol{V} ^\top  \boldsymbol{V}  \boldsymbol{M} ^\top+ \boldsymbol{M}  \boldsymbol{V} ^\top  \boldsymbol{M}  \boldsymbol{M} ^\top\\
		+ \boldsymbol{M}  \boldsymbol{M} ^\top  \boldsymbol{V}  \boldsymbol{V} ^\top+ \boldsymbol{M}  \boldsymbol{M} ^\top  \boldsymbol{V}  \boldsymbol{M} ^\top+ \boldsymbol{M}  \boldsymbol{M} ^\top  \boldsymbol{M}  \boldsymbol{V} ^\top
		\end{multline*}
		Taking the expectation and the trace and using $\tr \left[AB^\top CD^\top\right]=\tr \left[BA^\top DC^\top\right]\tr=\left[CD^\top AB^\top\right]=\tr \left[DC^\top BA^\top\right]$, we get for the contrast
		\begin{multline*}
		\mathbb{E}\tr \left( \boldsymbol{X}   \boldsymbol{X} ^\top  \boldsymbol{X}   \boldsymbol{X} ^\top \right)=\mathbb{E}_z\mathbb{E}_{ \boldsymbol{X} |z}\tr \left( \boldsymbol{X}   \boldsymbol{X} ^\top  \boldsymbol{X}   \boldsymbol{X} ^\top \right)\\=\sum \pi_k \left(\|\boldsymbol{\mu}_k\|^4+\mathbb{E}_{\boldsymbol{V}|z}\tr \left[\boldsymbol{V} \boldsymbol{V}^\top \boldsymbol{V} \boldsymbol{V}^\top \right]+4\boldsymbol{\mu}_k^\top\boldsymbol{\Sigma}_k\boldsymbol{\mu}_k+2\|\boldsymbol{\mu}_k\|^2\tr\left[\boldsymbol{\Sigma}_k\right]\right)
		\end{multline*}
		Since the $\boldsymbol{V}|z$ is gaussian, the fourth order cumulants are zero, which leads to a simplification of the second term of the last expression
		\[
		\mathbb{E}_{\boldsymbol{V}|z}\tr \left[\boldsymbol{V} \boldsymbol{V}^\top \boldsymbol{V} \boldsymbol{V}^\top \right]=\sum_k \pi_k \left(\tr\left[\boldsymbol{\Sigma}_k\right]^2+2\tr\left[\boldsymbol{\Sigma}_k^2\right]\right)
		\]
		We can notice that all terms but one are influenced by introducing a generic transformation $Y\mapsto UY$ with $U\in SO(p)$, hence the result.
\end{proof}